\newcounter{template}
\ifnum\value{template}=1
    \documentclass[sn-jnl,iicol]{sn-jnl}
    \jyear{2024}%
\else
    \documentclass[lettersize,journal]{IEEEtran}
\fi

\usepackage[T1]{fontenc}
\usepackage{graphicx} % for pdf, bitmapped graphics files
\usepackage{amsmath} % assumes amsmath package installed
\usepackage{amsfonts}       % blackboard math symbols
\usepackage{nicefrac}       % compact symbols for 1/2, etc.
\usepackage{diagbox}
\usepackage{subfig}			% per usare le subfigure
\usepackage{adjustbox}
\usepackage{amsthm}
\usepackage{flushend}
\usepackage[linesnumbered,ruled,vlined]{algorithm2e}

\newcommand{\fillbox}[3]% #1=width, #2=height, #3=filename
{\bgroup
  \dimen1=#1\relax% store width into register
  \dimen2=#2\relax% store height into register
  \sbox0{\includegraphics[width=#1]{#3}}%
  \ifdim\ht0>\dimen2
    \dimen0=\dimexpr \ht0-\dimen2\relax
    \adjustbox{clip=true,trim=0pt 0.5\dimen0 0pt 0.5\dimen0}{\usebox0}%
  \else
    \sbox0{\includegraphics[height=#2]{#3}}%
    \ifdim\wd0>\dimen1
      \dimen0=\dimexpr \wd0-\dimen1\relax
      \adjustbox{clip=true,trim=0.5\dimen0 0pt 0.5\dimen0 0pt}{\usebox0}%
    \else
      \usebox0
    \fi
  \fi
\egroup}

\newtheorem{definition}{Definition}

\newtheorem{remark}{Remark}
\newtheorem{lemma}{Lemma}

\DeclareMathOperator*{\argmin}{argmin}

\usepackage{eso-pic}% http://ctan.org/pkg/eso-pic
\AddToShipoutPictureBG{% Add picture to background of every page
  \AtPageLowerLeft{%
    \raisebox{63\baselineskip}{\makebox[\paperwidth]{\begin{minipage}{21cm}\centering
     AUTONOMOUS ROBOTS, SPRINGER NATURE. PREPRINT VERSION. ACCEPTED MARCH, 2024 \ \\
    \end{minipage}}}%
  }
}

\ifnum\value{template}=1
    \title[ ]{\bf{Adaptive Hybrid Local-Global Sampling for Fast Informed Sampling-Based Optimal Path Planning}}
    \author*[1]{\fnm{Marco} \sur{Faroni}}\email{marco.faroni@polimi.it}
    \author[2]{\fnm{Nicola} \sur{Pedrocchi}}\email{nicola.pedrocchi@stiima.cnr.it}
    \author[3]{\fnm{Manuel} \sur{Beschi}}\email{manuel.beschi@unibs.it}
    \affil[1]{\small \orgdiv{Dipartimento di Elettronica, Informazione e Bioingegneria}, \orgname{Politecnico di Milano}, \orgaddress{\street{Piazza Leonardo da Vinci 32}, \city{Milano}, \postcode{20133}, \country{Italy}}}
    \affil[2]{\small \orgdiv{STIIMA-CNR - Institute of Intelligent Industrial Technologies and Systems}, \orgname{National Research Council of Italy}, \orgaddress{\street{Via Alfonso Corti 12}, \city{Milano}, \postcode{20133}, \country{Italy}}}
    \affil[3]{\small \orgdiv{Dipartimento di Ingegneria Meccanica e Industriale}, \orgname{University of Brescia }, \orgaddress{\street{Via Branze 38}, \city{Brescia}, \postcode{25123}, \country{Italy}}}
\else
    \title{Adaptive Hybrid Local-Global Sampling for Fast Informed Sampling-Based Optimal Path Planning}
    \author{Marco Faroni$^{1}$, Nicola Pedrocchi$^{2}$, Manuel Beschi$^{3}$
    \thanks{$^{1}$ Dipartimento di Elettronica, Informazione e Bioingegneria, Politecnico di Milano, Milan, Italy.
    {\tt\small marco.faroni@polimi.it}}%     
    \thanks{$^{2}$ STIIMA-CNR - Institute of Intelligent Industrial Technologies and Systems, National Research Council of Italy 
    {\tt\small \{name.surname\}@stiima.cnr.it} }% 
    \thanks{$^{1}$ Dipartimento di Ingegneria Meccanica e Industriale, University of Brescia, Brescia, Italy.
    {\tt\small manuel.beschi@unibs.it}}% 
    }%
\fi

\begin{document}

\ifnum\value{template}=1
    \abstract{
    This paper improves the performance of RRT$^*$-like sampling-based path planners by combining admissible informed sampling and local sampling (\emph{i.e.}, sampling the neighborhood of the current solution). 
    An adaptive strategy regulates the trade-off between exploration (admissible informed sampling) and exploitation (local sampling) based on online rewards from previous samples. 
    The paper demonstrates that the algorithm is asymptotically optimal and has a better convergence rate than state-of-the-art path planners (e.g., Informed-RRT$^*$) in several simulated and real-world scenarios. 
    An open-source, ROS-compatible implementation of the algorithm is publicly available.
    }
\fi

\maketitle

\ifnum\value{template}=2
    \thispagestyle{empty}
    \pagestyle{empty}
    \begin{abstract}
    This paper improves the performance of RRT$^*$-like sampling-based path planners by combining admissible informed sampling and local sampling (\emph{i.e.}, sampling the neighborhood of the current solution). 
    An adaptive strategy regulates the trade-off between exploration (admissible informed sampling) and exploitation (local sampling) based on online rewards from previous samples. 
    The paper demonstrates that the algorithm is asymptotically optimal and has a better convergence rate than state-of-the-art path planners (e.g., Informed-RRT$^*$) in several simulated and real-world scenarios. 
    An open-source, ROS-compatible implementation of the algorithm is publicly available.
    \end{abstract}
\fi

\section{Introduction}
\label{sec: intro}

Path planning is a fundamental problem in robotics and with a heavy impact on a broad variety of applications. 
For example, the recent developments in humanoid robotics require fast planning tools to handle high-dimensional systems. 
Similarly, industrial and service robotics often deal with dynamic environments where the robot must plan the motion on the fly. 
An example is a robot arm that picks objects from a conveyor belt or cooperates with humans to assemble a piece of furniture. 
A common thread of these applications is the high dimensionality of the search space and the limited available computing time to find a solution.

Path-planning problems are solved mainly through graph-based or sampling-based approaches.
Graph-based methods \cite{A_star,graph_based} are used mainly for navigation problems, while sampling-based methods are the most widespread in robotic manipulation because they are more efficient with high-dimensional systems. 
Sampling-based methods explore the search space by randomly sampling the robot configuration space to find a sequence of feasible nodes from start to goal. 
Different strategies for sampling and connecting nodes have given birth to different algorithms, such as RRT \cite{Lavalle-RRT}, EST \cite{Latombe-EST}, and PRM \cite{Kavraki-PRM}.

Sampling-based methods are successful in robotics because they do not require discretizing the search space, do not explicitly require the construction of the obstacle space, and generalize well to different robots' structures and specifications. 
These advantages come at the cost of weaker completeness and optimality guarantees. 
In particular, they can provide \emph{asymptotic optimality}; that is, the probability of converging to the optimal solution approaches one as the number of samples goes to infinity \cite{karaman-RRT*}. 
The convergence rate of such algorithms is relatively slow, and actual implementations usually stop the search way before they reach the optimum. 
A meaningful improvement to optimal planners came with the introduction of informed sampling \cite{Gammel-InformedRRT}. 
Informed sampling-based planners shrink the sampling space 
every time the solution cost decreases, making the convergence to the optimal solution faster. 
These planners show a slow convergence rate when the cost heuristic is poorly informative. 
In the case of path length minimization, the Euclidean distance can be chosen as a heuristic of the cost between two points. 
However, with many obstacles, there is a large difference between the Euclidean distance and the actual minimum path between the two points. 
In these cases, the convergence speed resembles that of uninformed planners (\emph{e.g.}, RRT$^*$ \cite{karaman-RRT*}). 
This paper tackles this issue by proposing a mixed strategy that alternates sampling the informed set and the neighborhood of the current solution. 
The rationale is that the cost of the solution improves by sampling its neighborhood (\emph{i.e.}, local sampling), with a consequent quick reduction of the measure of the informed set.

Alternating admissible and locally informed sampling is an example of the classic exploration-versus-exploitation dilemma, which is hardly solvable with a fixed ratio between the usage of the two sampling strategies.  
To overcome this issue, we propose an adaptive technique to dynamically balance the choice of one sampling strategy over the other. 
The result is that the search algorithm prefers exploitation (\emph{i.e.}, local sampling) only as long as it is useful and switches to exploration (\emph{i.e.}, admissible informed sampling) to avoid stagnation.

The paper's contribution is twofold.  
First, it defines a mixed sampling strategy that combines global and local informed sampling for asymptotically optimal sampling-based path planners. 
Local informed sampling oversamples the neighborhood of the current solution to quickly reach a local optimum, while global informed sampling guarantees asymptotic optimality. 
Second, It proposes an asymptotically optimal algorithm that uses the mixed sampling strategy and dynamically adjusts the trade-off between global and local sampling, showing that this outperforms state-of-the-art planners, such as Informed-RRT$^*$, on different classes of problems.

An open-source ROS-compatible version of the planner is publicly available \cite{git_mirrt}. 

The paper is organized as follows. 
Section \ref{sec: informed-planning} introduces the reader to optimal planning and informed sampling. 
Section \ref{sec: related_works} discusses previous works on the acceleration of informed sampling-based planners.
Section \ref{sec: motivation} discusses the motivation of this work through some illustrative examples. 
Section \ref{sec: method} describes the proposed method.
Section \ref{sec: results} compares it with other methods. 
Section \ref{sec: conclusions} concludes and discusses future works.

\section{Informed sampling-based optimal path planning}
\label{sec: informed-planning}

This section introduces the concepts of path planning, informed sets, and informed sampling used throughout the paper.

The path planning problem is formulated in the configuration space, $X \subseteq \mathbb{R}^n$, which denotes all possible configurations $x$ of the system (for robot manipulators, $x$ is usually a vector of joint angles). 
Let $X_{\mathrm{obs}}$ be the space of all those configurations in collision with an obstacle, and $X_{\mathrm{free}}=\mathrm{cl}( X \setminus X_{\mathrm{obs}})$ the obstacle-free configuration space, where $\mathrm{cl(\cdot)}$ denotes the closure of the set. 
\begin{definition}{\emph{(optimal path planning)} [adapted from  \cite{Gammel-InformedRRT}]}
\label{def: optimal-motion-planning}
    Given a starting point $x_{\mathrm{start}}$ and a set of desired goal points $X_{\mathrm{goal}} \subset X$, \emph{optimal path planning} is the problem of finding a curve $\sigma^* \,: [0,1] \rightarrow X_{\mathrm{free}}$ such that: 
    \begin{equation}
        \sigma^* = \argmin_{\sigma \in \Sigma} \big\{ c(\sigma) \, \vert \, \sigma(0) = x_{\mathrm{start}}, \sigma(1) \in X_{\mathrm{goal}} \big\}
    \end{equation}
    where $c: \,\Sigma \rightarrow \mathbb{R}_{\geq 0}$ is a Lipschitz continuous cost function associating a cost $c(\sigma)$ to a curve $\sigma \in \Sigma$, $\Sigma$ is the set of solution paths, and $\mathbb{R}_{\geq 0}$ is the set of non-negative real numbers.
\end{definition}
\begin{remark}
Cost function $c$ is often the length of the path so that the optimal motion plan is the shortest collision-free path from $x_{\mathrm{start}}$ to $X_{\mathrm{goal}}$. 
\end{remark}
\noindent If an algorithm can find a solution to the optimal path planning problem, then it is said to be an \emph{optimal path planner}. 
Sampling-based path planners, such as RRT$^*$ \cite{karaman-RRT*}, can only ensure the probabilistic convergence to the optimal solution. 
This weaker form of optimality is referred to as \emph{(almost-sure) asymptotic optimality}.
The convergence rate of an asymptotically optimal planner is related to the probability of sampling points that can improve the current solution. 
This set is referred to as the \emph{omnisicient set} \cite{Gammel-InformedRRT}. 
RRT$^*$ and similar algorithms, as proposed in \cite{karaman-RRT*}, are very inefficient at sampling the omniscient set (the probability that RRT$^*$ samples a point that belongs to the omniscient set decreases factorially in the state dimension \cite{Gammel-InformedRRT}). 
To increase the probability of sampling the omniscient set, Gammell et al. \cite{Gammel-InformedRRT} coined the concept of informed sampling; that is, sampling an approximation of the omniscient set (the \emph{informed set}) so that the probability of finding a point that improves the current solution is higher. 
If the informed set is a superset of the omniscient set, it is referred to as an \emph{admissible informed set}. 
\begin{definition}{\emph{(admissible informed set)}[adapted from  \cite{Gammel-InformedRRT}]}
\label{def: admissible-informed-set}
    An informed set $X_{\hat f}$ is a heuristic estimate of the omniscient set $X_f$.
    If $X_{\hat{f}} \supseteq X_f$, the informed set is said to be \emph{admissible}.
\end{definition}
\noindent In minimum-length path planning, it is always possible to construct an admissible informed set by considering that the shortest path through a sample $x \in X$ is lower bounded by the sum of Euclidean distances from $x_{\mathrm{start}}$ to $x$ and from $x$ to $x_{\mathrm{goal}} \in X_{\mathrm{goal}}$).
As a consequence, all possibly improving points lie in the so-called $\mathcal{L}_2$-informed set, $X_{\hat{f}}$, given by: 
\begin{equation}
\label{eq: informed-set}
    X_{\hat{f}} = \big\{ x \in X_{\mathrm{free}}\, \vert\, \| x-x_{\mathrm{start}}\|_2 + \|x_{\mathrm{goal}} - x\|_2 < c_k \big\}
\end{equation}
where $c_k$ is the cost of the best solution at iteration $k$.
Notice that such an informed set is equivalent to the intersection of the free space $X_{\mathrm{free}}$ and an $n$-dimensional hyper-ellipsoid symmetric about its transverse axis with focal points at $x_{\mathrm{start}}$ and $x_{\mathrm{goal}}$, transverse diameter equal to $c_k$, and conjugate diameters equal to $\sqrt{c_k^2-c_{\mathrm{min}}^2}$, where 
\begin{equation}
    c_{\mathrm{min}}=\| x_{\mathrm{goal}} - x_{\mathrm{start}}\|_2.
\end{equation}
The volume of the hyper-ellipsoid decreases progressively as the solution cost $c_k$ decreases, improving the convergence rate of the algorithm.

\section{Related works}
\label{sec: related_works}

Informed sampling stems from the simple but effective idea of sampling only points with a higher probability of improving the solution. 
This is not a new idea, in principle, as several works use heuristics to bias sampling \cite{Urmson2003, Amato:resample, Salzman2013, Shan2014, Ge2016, SantanaCorreia2018, Yu2019, Francis-bayesian-rrt,Faroni_RAL2023}.

The main issue with sampling bias is that, depending on the geometry of the problem, the heuristic may discard points of $X_f$.  \cite{Gammel-InformedRRT}.
This can be deleterious for the convergence speed, and it may even compromise the optimality of the algorithm. 
Compared to these works, admissible informed sampling never excludes any points possibly belonging to the omniscient set; thus, it retains asymptotic optimality regardless of the geometry of the problem.  
Nonetheless, convergence speed may be slow when the admissible heuristic is not informative.

Few works attempted to speed up the convergence rate by combining informed planning and local techniques. 
In \cite{Kim-Song-wrapping-informed-rrt-2015, Kim-Song-wrapping-informed-rrt-2018}, Kim and Song propose to run a deterministic path short-cutter every time the algorithm improves the solution. 
The short-cutting procedure acts as follows: i) it considers three consecutive nodes on the path at a time; ii) it discretizes the two corresponding edges; iii) it tries to connect the extreme nodes to the sampled segment until it finds a collision; iv) it moves the central node to the intersection of the two segments found in the previous step. 
Such an approach has two main drawbacks. 
First, the computational time owed to the short-cutting is significant as it requires an iterative edge evaluation (\emph{i.e.}, collision checking) every time it tries to refine a triple of nodes. 
Second, this approach is suitable only for minimum-path problems, as it relies on the triangular inequality applied to each triple of nodes.
\cite{hauer-DRRT} proposes to refine the current solution by moving the nodes of the tree based on gradient descent. 
Yet, every time a node is moved, the refinement process requires an intensive edge evaluation.

The idea of combining global and local optimization was also explored by Choudhury et al. \cite{Gammell-RABIT}, who propose a hybrid use of BIT$^*$ \cite{Gammell2015-bit}, a lazy heuristic-driven informed planner and CHOMP \cite{ratliff-chomp}, a gradient-based local planner. 
Roughly speaking, the local planner is used to solve a two-point problem between a pair of nodes. 
One main drawback is that the local planner is called every time an edge is evaluated, which may be computationally counter-effective. 
Other variants of BIT$^*$ were proposed in \cite{Gammell-AIT}, and \cite{Gammell-ABIT}, but focusing on how to improve the heuristics by experience.
\cite{Faroni_RAL2023} uses online learning (clustering of previous edges and Multi-Armed Bandits) to oversample promising regions.

Finally, \cite{Panagiotis-relevant-region} and \cite{Srinivasa-GUILD} propose approaches to focus the search on subsets of the informed set.
\cite{Srinivasa-GUILD} decomposes a planning problem into two sub-problems and applies informed sampling to them. 
The union of the informed sets of the sub-problems is strictly contained in the informed set of the initial problem; thus, the search focuses on a smaller region. 
However, the performance of such an approximation strongly depends on the problem geometry, and the authors do not discuss how to retain asymptotic optimality.
\cite{Panagiotis-relevant-region} estimates the cost-to-come of the tree leaves to bias the search towards a subset of the informed set, called the relevant region. 
In this case, the trade-off between exploration and exploitation is fixed. 
Thus, the performance depends on the problem geometry, and it may be even worse than admissible informed sampling.

Our approach is similar to the works mentioned above as it alternates informed sampling and local refinement of the path. 
Compared to \cite{Kim-Song-wrapping-informed-rrt-2015, Kim-Song-wrapping-informed-rrt-2018} and \cite{hauer-DRRT}, our method refines the path by sampling the neighborhood of the current solution, and this allows for gradient-free refinement, also with generic cost functions. 
Moreover, \cite{Kim-Song-wrapping-informed-rrt-2015,Kim-Song-wrapping-informed-rrt-2018} and \cite{hauer-DRRT} tend to favor exploitation (\emph{i.e.}, path refinement) rather than exploration, wasting time optimizing suboptimal solutions (see numerical results in Section \ref{sec: results}). 
Similarly, \cite{Gammell-RABIT} and \cite{Panagiotis-relevant-region} use a fixed balance between exploration and exploitation; thus, the performance may vary a lot across different problems. 
Our method adjusts the trade-off between exploration and exploitation according to the cost progression, adapting to different problems. 
In this sense, we could use our adaptive scheme in \cite{Panagiotis-relevant-region} to dynamically balance the trade-off between exploration and exploitation and in \cite{Srinivasa-GUILD} to retain asymptotic optimality.

\section{Motivation for an adaptive mixed sampling strategy}
\label{sec: motivation}

\begin{figure*}[tpb]
	\centering
	\subfloat[][Large obstacle between start and goal.]
	{\includegraphics[height=2.4cm,width=0.3\textwidth]{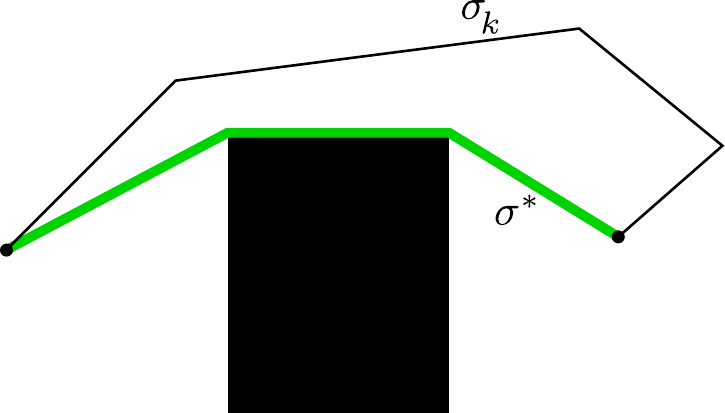} \label{fig: examples-a}}\,\,%
	\subfloat[][Narrow passage, current solution stacked on a local minimum.]
	{\includegraphics[height=2.4cm,width=0.3\textwidth]{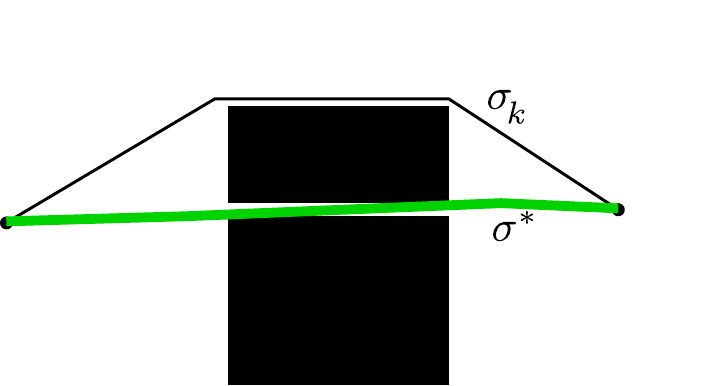}\label{fig: examples-b}}\,\,%
	\subfloat[][Narrow passage, current solution far from local minima.]
	{\includegraphics[height=2.4cm,width=0.3\textwidth]{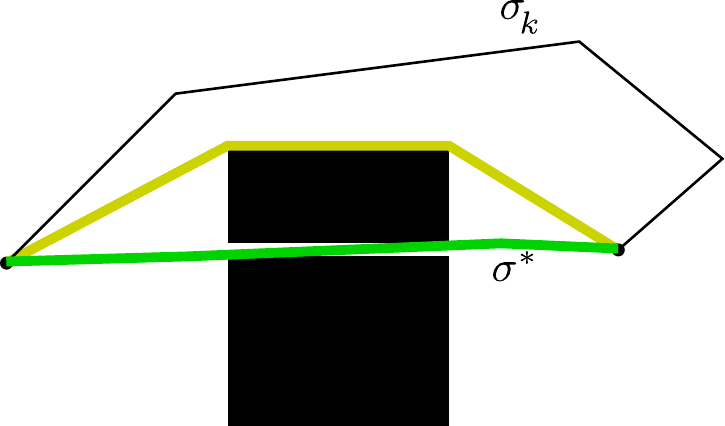}\label{fig: examples-c}}%
	\caption{Examples of planning situations where local sampling is useful (a), deleterious (b), or unable to find the global optimum but useful to reduce the measure of the informed set (c).}
	\label{fig: examples}
\end{figure*}

To understand the motivation behind this work, consider the minimum-path problem in Figure \ref{fig: examples-a}. 
Because of the presence of a large obstacle between $x_{\mathrm{start}}$ and $x_{\mathrm{goal}}$, the $\mathcal{L}_2$-informed set is large and poorly informative. 
Sampling the neighborhood of the current solution would be much more efficient than considering the whole $\mathcal{L}_2$-informed set, as the path would quickly converge to the global optimum. 
This situation is expected when the current and the optimal solutions are homotopic. 
We will refer to this sampling strategy as \emph{local sampling}. 
%The following proposition summarizes the benefits of local sampling.
%
%\begin{proposition}[Benefit of local sampling]
%If the current solution $\sigma_k$ is not locally optimal, the portion of the neighborhood of $\sigma_k$ in $X_{\mathrm{free}}$ where the functional derivative to $\sigma$ of the cost function is negative belongs to the omniscient set.
%\end{proposition}

On the other hand, when the current solution is locally optimal, any efforts on the local optimization would be useless.
For example, in Figure \ref{fig: examples-b}, the optimal solution passes through the narrow passage between the two obstacles; thus, sampling the neighborhood of the current solution would lead to a local optimum (yellow in Figure \ref{fig: examples-c}).
As the solution approaches the local optimum, the probability of improving the solution via local sampling is equal to zero.
%
%\begin{proposition}[Drawback of local sampling]
%If the current solution $\sigma_k$ is locally optimal, then a neighborhood of $\sigma_k$ exists, and the intersection of the neighborhood and the omniscient set is empty.  An algorithm that samples such a neighborhood has a null convergence rate.
%\end{proposition}

Notice that a fast convergence to a local minimum quickly reduces the volume of the informed set. 
However, it is crucial to understand when the local sampling is beneficial without losing the asymptotic global optimality.

In this paper, we combine admissible informed sampling and local sampling in a \emph{mixed sampling strategy}.
%\footnote{The term \emph{mixed strategy} is borrowed from game theory, where it means that an agent chooses what action to take based on a probability assigned to each possible action. 
%We can think of a path planning problem as a game against nature where the player (\emph{i.e.}, the planning algorithm) randomly chooses the strategy (local or informed sampling) based on an assigned probability.}.
On the one hand, sampling the admissible informed set guarantees that all points from the omniscient set are taken into account. 
On the other hand, local sampling has a twofold role. 
First, if the local and the global optima correspond, it quickly converges to the solution, as in Figure \ref{fig: examples-c}. 
Second, it reduces the size of the admissible informed set. 
Indeed, the Lebesgue measure of the $\mathcal{L}_2$-informed set is directly related to the best cost to date $c_k$ as follows: 
\begin{equation}
    \lambda(X_{\hat{f}}) = c_k\, (c_k^2 - c_{\mathrm{min}}^2)^{\frac{n-1}{2}}\,\frac{\zeta_n}{2^n} 
\end{equation}
where $\zeta_n$ is the Lebesgue measure of the unit ball (dependent only on $n$) \cite{Gammel-InformedRRT}.
Hence, improving the current solution (even in the neighborhood of a local minimum) enhances the convergence speed to the globally optimal solution.

\section{Proposed approach}
\label{sec: method}

This section describes the proposed mixed sampling strategy. First, it defines the local informed set. Second, it designs an algorithm to dynamically change the local sampling probability based on the cost evolution. Finally, it proves asymptotic optimality. 

\subsection{Mixed-strategy sampling}
\label{sec: local_sampling}
Consider an $n$-dimensional path planning problem solved by a sampling-based planner.
Let $\sigma_k\in X_{\mathrm{free}}$ be the current solution at iteration $k$ and $c_k=c(\sigma_k)$.
An RRT$^*$-like planner is asymptotically optimal if the algorithm that connects nodes satisfies conditions on the minimum rewire radius and the sampler draws nodes from a superset of the omniscient set. 
If we drop the second condition, such a relaxed planner would converge to a local optimum.
To formulate this idea more formally, we introduce the notion of \emph{local informed set}. 
Then, we combine it with admissible informed sampling to obtain the adaptive \emph{mixed-strategy sampler} used in the proposed planner.

\begin{definition}{\emph{(local informed set)}}
    The \emph{local informed set} of the current solution $\sigma_k$ is the intersection of the admissible informed set and the set of points with distance  smaller than $R$ from $\sigma_k$: 
    \begin{equation}
    \label{eq: local-informed-set}
    X_{\hat{f},l} = \left\{ x\in X_{\hat{f}} \;\Big\vert\;   \min_{y \in \sigma_k} \Big( \left\Vert x-y \right\Vert_2 \Big) < R \right\}.
\end{equation}
\end{definition}

\begin{lemma}{\emph{(local optimality of local sampling)}}
Consider an asymptotically optimal path planner and let the sampling algorithm draw samples only from the local informed set. 
The planner converges to a local minimum with a probability equal to one.
\end{lemma}
\begin{proof}: 
If the current solution is not a local optimum, the intersection of the omniscient set and any neighborhoods of $\sigma_k$ is not empty ($c$ is Lipschitz):
\begin{equation*}
    X_f \cap X_{\hat{f},l} \neq \emptyset\quad\mbox{with}\, X_{\hat{f},l}\,\mbox{neighborhood of}\,\sigma_k
\end{equation*}
It follows that local sampling improves the solution with a probability greater than zero whenever the solution is not (locally) optimal.
\end{proof}

We define hereafter a mixed-strategy sampler to combine admissible and locally informed sampling soundly.

\begin{definition}{\emph{(mixed-strategy sampler)}}
A \emph{local sampler} and a \emph{global sampler} are algorithms that draw samples from $X_{\hat{f},l}$ and  $X_{\hat{f}}$, respectively.
A \emph{mixed-strategy sampler} draws samples by using a local sampler with probability $\phi$ and a global sampler with probability $1-\phi$.
\end{definition}

\begin{remark}
A mixed-strategy sampler is admissible if $\phi < 1$. 
\end{remark}

\begin{lemma}{\emph{(optimality of admissible mixed-strategy samplers)}} \label{th: asympotic-optimality}
A sampling-based path planner that is asymptotically optimal under uniform sampling distribution is asymptotically optimal also under admissible mixed-strategy sampling.
\end{lemma}
\begin{proof}:
A mixed-strategy sampler samples $X_{\mathrm{free}}$ with non-uniform probability density $d$.
If such a sampler is admissible, $d$ can be seen as a mixture of probability densities such that: 
\begin{equation*}
    d = (1-\phi) \, d_1 + \phi\,  d_2 
\end{equation*}
where $d_1$ is a strictly positive uniform probability density over $X_{\mathrm{free}}$ and
\begin{equation*}
    d_2 = \frac{d\, \lambda(X_{\mathrm{free}}) - 1 + \phi}{\phi\, \lambda(X_{\mathrm{free}})}.
\end{equation*}
Based on this consideration, the asymptotic optimality of the path planner traces back to the proof of asymptotic optimality of \cite{Pavone-FMT} with non-uniform sampling. 
In particular, the planner is still asymptotically optimal by adjusting the rewire radius of a factor $(1-\phi)^{-1/n}$, as proved in Appendix D of \cite{Pavone-FMT}.
\end{proof}

At each iteration, the mixed-strategy sampler should select an appropriate value of $\phi$ based on the likelihood of improving the current solution. 
This is important to exploit the advantages of both admissible and local informed sampling (respectively, global asymptotic optimality and fast convergence to local optima) and mitigate the flaws (slow convergence speed and stagnation into local minima).  
We denote the guess that $\sigma_k$ is not a local optimum at iteration $k$ by $p_k\in[0,1]$. 
If $c_k<c_{k-1}$, we increase $p_{k+1}$  proportionally to the relative improvement of the cost such that: 
\begin{equation}
p_{k+1}=\nu p_k+ \left(1-\nu\right) \frac{c_{k-1}-c_{k}}{c_{k-1}-u}
\label{eq: guess}    
\end{equation}
where $\nu\in[0,1)$ is a forgetting factor that smooths the evolution of $p$ and $u$ is an admissible estimate of the best cost $c^*$.
Note that the cost $c_k$ is non-increasing (namely, $c_k \leq c_{k-1}$), therefore $p_k$ is a strictly positive number (assuming $p_0>0$). %
Moreover, $p_{k+1}\leq 1$ because $c_k\leq u$ and $p_0\leq1$.

It follows that a selector that uses $\phi = p_k$ 
is admissible.

\subsection{Proposed algorithm}
\label{subsec: proposed-algorithm}

The proposed planner is the variant of Informed-RRT$^*$ in Algorithm \ref{alg: algorithm}. 
It uses the guess $p_k$ as the probability to sample the local informed set (lines 1--6). 
Sample $x$ is used to extend the tree (line 7); and $p_{k+1}$ is updated according to \eqref{eq: guess} (lines 8 and 9).

\begin{algorithm}[tpb]
\KwIn{$p_k$, $\sigma_{k-1}$, $c_{k-1}$, $\nu$,  $R_0$}
\KwOut{$p_{k+1}$, $\sigma_k$, $c_{k}$}

  $j=\texttt{rand.uniform}([0,1])$\;
  \uIf{$j<p_k$}{
    $R = R_0 (c_k - u)$\;
    $x=\texttt{localSampling}(\sigma_{k-1},R)$\;
  }
  \Else{
    $x=\texttt{informedSampling}(\sigma_{k-1},c_k)$\;
  }
  $(\sigma_k,c_k)=\texttt{connectAndRewire}(x)$\;
  \uIf{$c_k<c_{k-1}$}{
    $p_{k+1}=\nu\, p_k+ \left(1-\nu\right) \dfrac{c_{k-1}-c_{k}}{c_{k-1} - u}$
  }
  \Else{
    $p_{k+1}=\nu\, p_k$
  }
 \caption{Mixed-strategy informed planner}
 \label{alg: algorithm}
\end{algorithm}

Procedures $\texttt{informedSampling}$ and $\texttt{localSampling}$ sample the admissible informed set \eqref{eq: informed-set} and the local informed set \eqref{eq: local-informed-set}, respectively. 
The former follows the implementation of \cite{Gammel-InformedRRT}, and the latter uses Algorithm \ref{alg: localsampling}.

Algorithm \ref{alg: localsampling} randomly samples a ball of radius $R$ centered at a random point along the current solution path. 
First, it uniformly samples the $n$-dimensional unit ball and assigns the value to $b$ (lines 2--4). 
Then, it picks a random point, $\sigma(s)$, on path $\sigma$. 
Therefore, the final candidate sample is obtained by scaling $b$, from the unit ball to the ball of radius $R$ and centered in $\sigma(s)$ (lines 5--6). 
Finally, it uses rejection to ensure that $x \in X_{\hat{f}}$.
Note that the rejection of the candidate is unlikely if $R$ is small. 

Algorithm \ref{alg: localsampling} does not sample the local informed set uniformly. 
The points closer to the path have a higher probability of being sampled than points near the boundary of the tube. 
Moreover, Algorithm \ref{alg: localsampling} over-samples regions ``inside'' the corners of the path. 
Non-uniform local sampling does not affect the asymptotic optimality of the planner (Lemma \ref{th: asympotic-optimality}). 
Moreover, in minimum-length problems, over-sampling regions inside the corners may be beneficial in reducing the path length.

\begin{algorithm}[tpb]
  \KwIn{curve $\sigma$, radius $R$ of the local informed set}
  \KwOut{sample $x\in X_{\hat{f},l}$}
   \Repeat{$x\in X_{\hat{f}}$}{
   $w = \texttt{rand.normal}(0,1)$\;
   $r = \texttt{rand.uniform}([0,1])$\;
   $b= \left(\sqrt[n]{r} \frac{w}{\|w\|_2}\right)$\;
   $s = \texttt{rand.uniform}([0,1])$\;
   $x= \sigma(s)+ R\, b$\;
      }
   \caption{$\texttt{localSampling}$ procedure}
   \label{alg: localsampling}
\end{algorithm}

\subsection{Algorithm tuning and convergence performance}
\label{sec: tuning}
Algorithm \ref{alg: algorithm} has two parameters more than Informed-RRT$^*$: the radius $R_0$ and the forgetting factor $\nu$. 
Appendix \ref{sec: Appendix} provides an illustrative example showing the effect of the parameters on the convergence.
Summarizing the results, $R_0 \in [0.01,0.02]$ and $\nu \approx 0.999$ consistently provide the best results across problems of different dimensionality and geometry.

\section{Experiments}
\label{sec: results}

We test our Mixed-strategy Informed planner (MI-RRT$^*$) with robot manipulators (6, 12, 18 degrees of freedom), navigation of mobile manipulators, and a real manufacturing case study. 
We demonstrate that MI-RRT$^*$ consistently outperforms the baselines.

\subsection{Robot manipulators}
\label{sec: manipulators}

\begin{figure*}[tbp]
	\centering
	\subfloat[][$n=6$]
	{\includegraphics[width=0.25\textwidth]{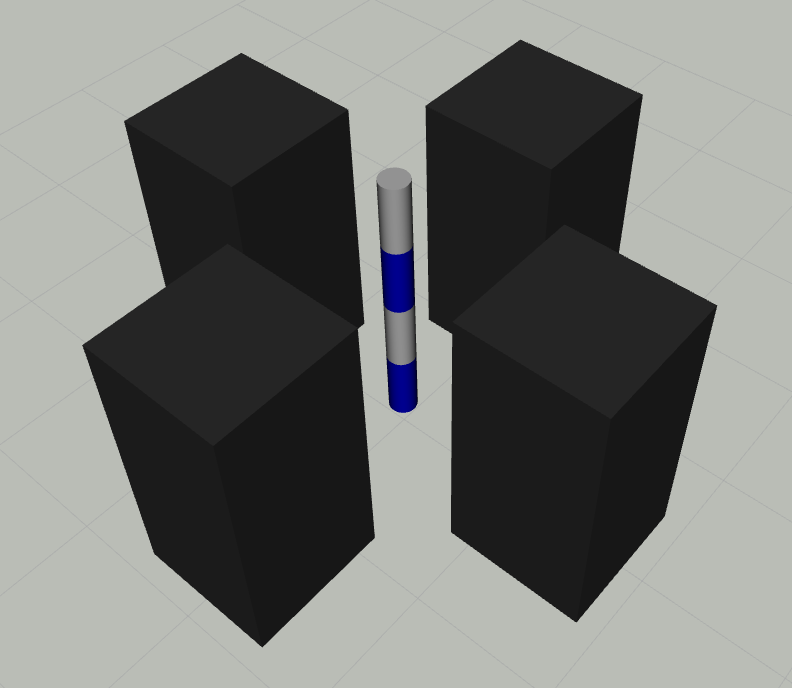} \label{fig: cells.a}}\,\,\,%	
    \subfloat[][$n=12$]
	{\includegraphics[width=0.25\textwidth]{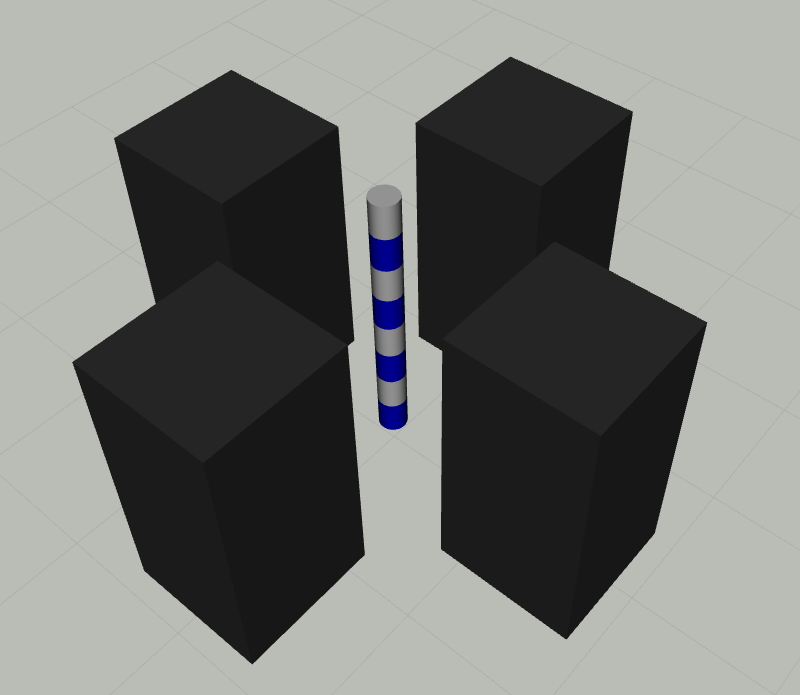} \label{fig: cells.c}}\,\,\,%
	\subfloat[][$n=18$]
	{\includegraphics[width=0.25\textwidth]{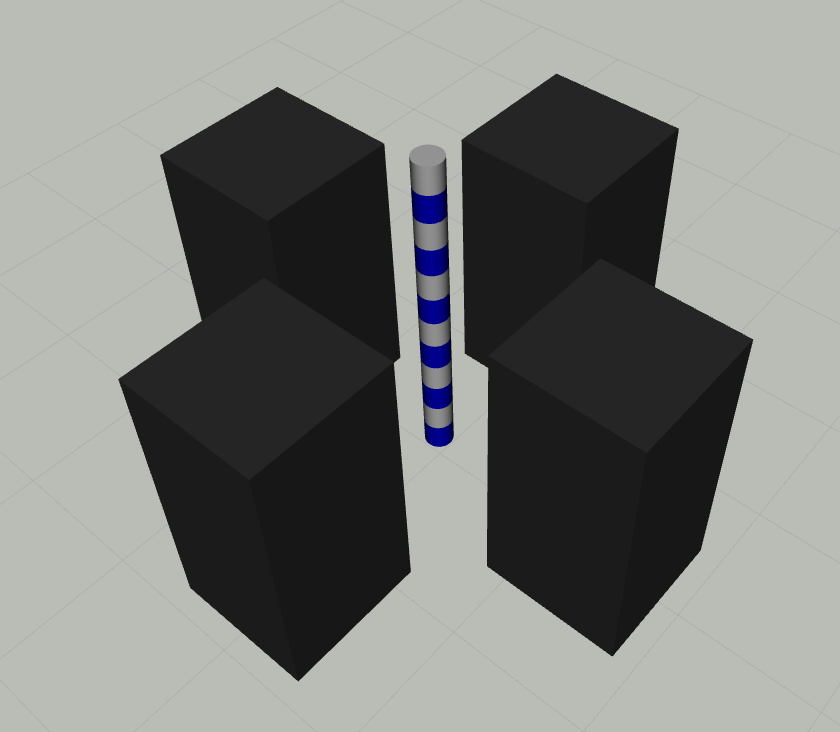} \label{fig: cells.d}}%
	\caption{Robot manipulator benchmark. $n$ is the robot's number of degrees of freedom.}
	\label{fig: cells}
\end{figure*}

%This section analyzes the proposed algorithm performance in realistic robotic cells and compares the algorithm with others from the literature. 
%The experiments test the algorithm with different values of $n$ and complex obstacle geometry, like narrow passages. 

We consider three robotic cells (Figure \ref{fig: cells}). 
Each cell has four rectangular obstacles and a serial manipulator (6, 12, and 18 degrees of freedom, respectively). 
%The robots are composed of several segments, and each segment is a three-degree-of-freedom anthropomorphic chain. 
The cell descriptions and usage examples are available at \cite{high_dof_cell_github}.

We compare our planner (MI-RRT$^*$) with Informed-RRT$^*$ \cite{Gammel-InformedRRT}, which uses a pure admissible informed sampling method, and wrapping-based Informed-RRT$^*$ (Wrap-RRT$^*$) \cite{Kim-Song-wrapping-informed-rrt-2018}, which applies a shortcutting procedure whenever it improves the solution.
The additional parameters of MI-RRT$^*$ are tuned according to Section \ref{sec: tuning}, namely $R=0.02(c_k-u)$ and $\nu=0.999$.

First of all, we show an example of a query to illustrate the behavior of the algorithms. 
Figure \ref{fig: high_dof_comparison_dof6} shows the cost trend for a random planning query with $n=6$, repeated 30 times for each planner. 
%We compute the median, the 10\%-, 90\%-percentiles of the best cost in each iteration. 
MI-RRT$^*$ provides a faster convergence rate and a smaller variance. 
Moreover, the median cost of the proposed algorithm is closer to the 10\%-percentile than the other strategies, highlighting the capability of MI-RRT$^*$ to converge sooner to the global minimum. 
The same behavior is clear also for $n=12$, as shown in Figure \ref{fig: high_dof_comparison_dof12}. 
In this case, Informed-RRT$^*$ suffers more from the curse of dimensionality, while Wrapping-based RRT$^*$ gets stuck in a local minimum for several iterations.

\begin{figure*}[tbp]
	\centering
	\subfloat[][$n=6$]
	{\includegraphics[height=4cm,width=0.35\textwidth]{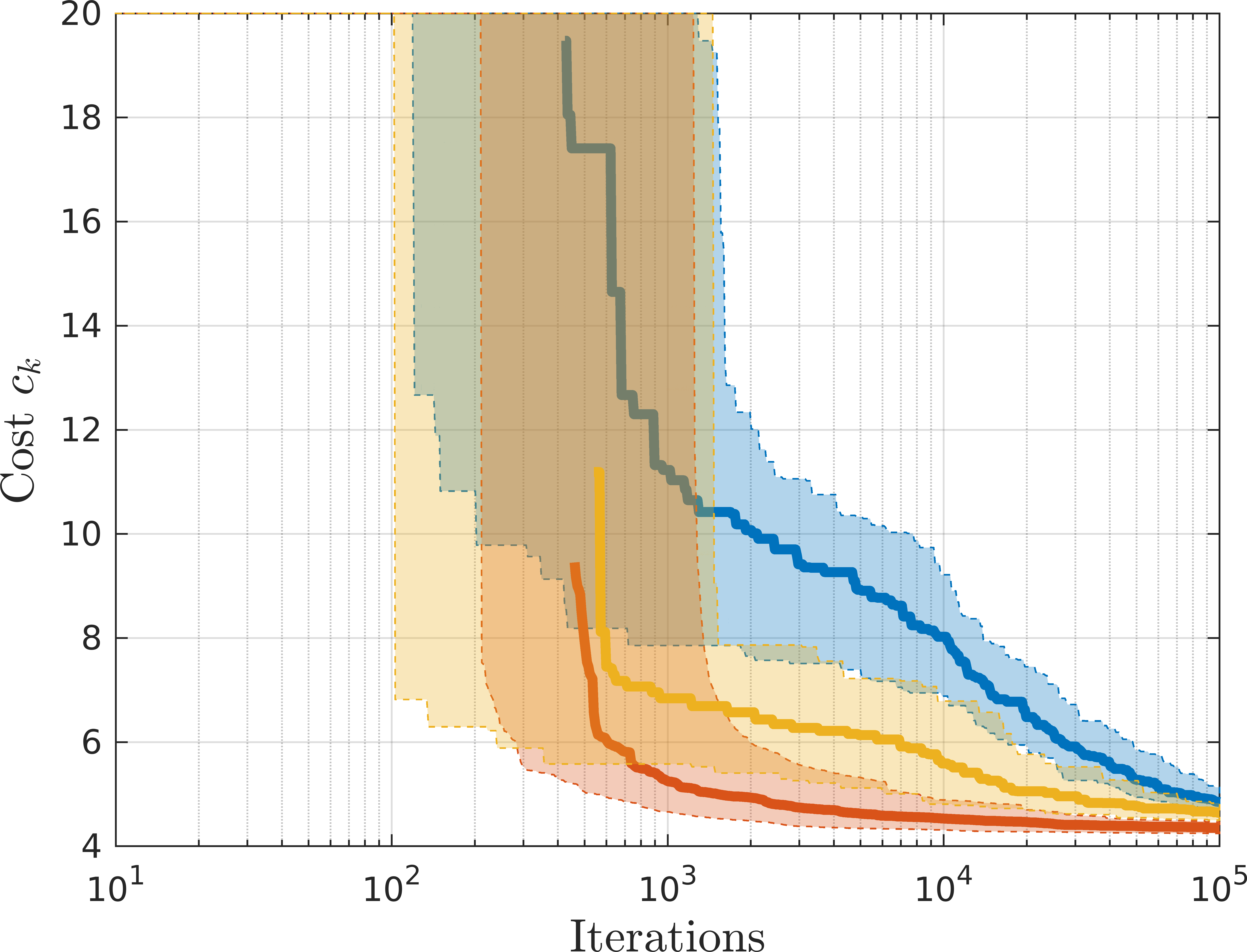} \label{fig: high_dof_comparison_dof6}}\,\quad%
	\subfloat[][$n=12$]
	{\includegraphics[height=4cm,width=0.35\textwidth]{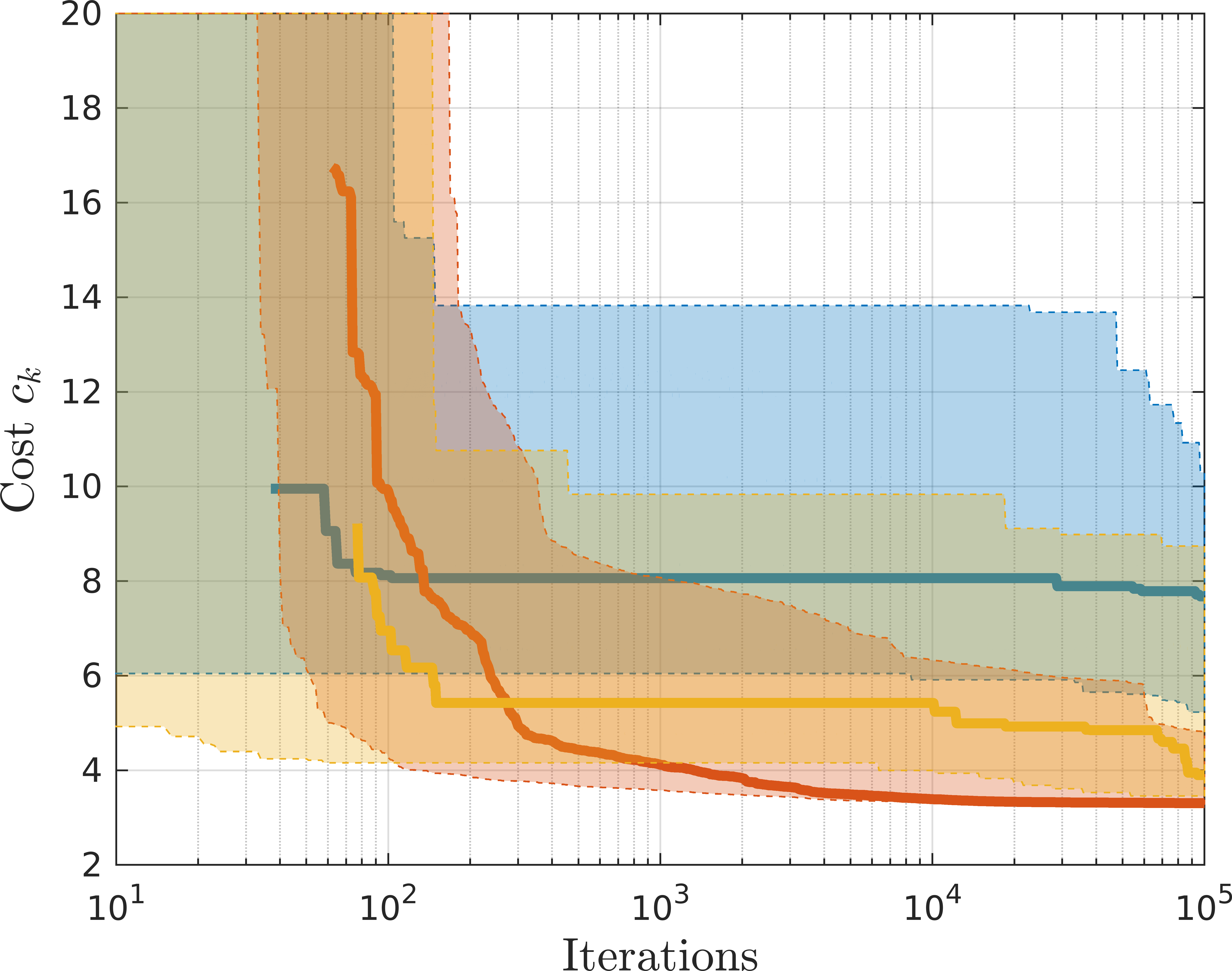} \label{fig: high_dof_comparison_dof12}}%
	\caption{Cost trends over iterations for the 6-dof and 12-dof manipulators for one query. Blue: Informed-RRT$^*$; Orange: Wrap-RRT$^*$; Red: Mixed-strategy RRT$^*$. Solid lines denote the median; shaded regions denote the region between the 10\% and 90\% percentile.
    }
	\label{fig: high_dof_comparison}
\end{figure*}

For an exhaustive comparison, we set up a benchmark as follows. 
Thirty queries are generated randomly (queries for which a direct connection between start and goal exists are discarded). 
The queries are solved with different planning times, between 0.5 and 5 seconds. 
Bounding the maximum planning time instead of the maximum number of iterations has been preferred because the algorithms perform different operations during the iterations. 
Moreover, planning time is more meaningful in practical applications.

Each planner solves each query 30 times for maximum planning times equal to ${0.5, 1.0, 2.0, 5.0}$ seconds. 
The final cost of each query is normalized by an estimate of the minimum cost, obtained by solving the query with a maximum planning time equal to 60 seconds.

The box-plots of  Figure \ref{fig: box-simulation} show that MI-RRT$^*$ has a faster convergence rate as well as a smaller variance compared to both Informed-RRT$^*$ and Wrap-RRT$^*$.
Therefore, the proposed approach finds better and more repeatable paths given the same amount of time. 
This result is emphasized for larger values of $n$, as shown in Figures \ref{fig: box_dof12}-\ref{fig: box_dof18}.

We did not observe significant differences between Wrap-RRT$^*$ and Informed-RRT$^*$, probably because the improvement of the convergence rate is counterbalanced by the computational overload owed to the wrapping procedures, as mentioned in Section \ref{sec: related_works}.

\begin{figure*}[tbp]
	\centering
	\subfloat[][$n=6$]
	{\includegraphics[height=4cm,width=0.32\textwidth]{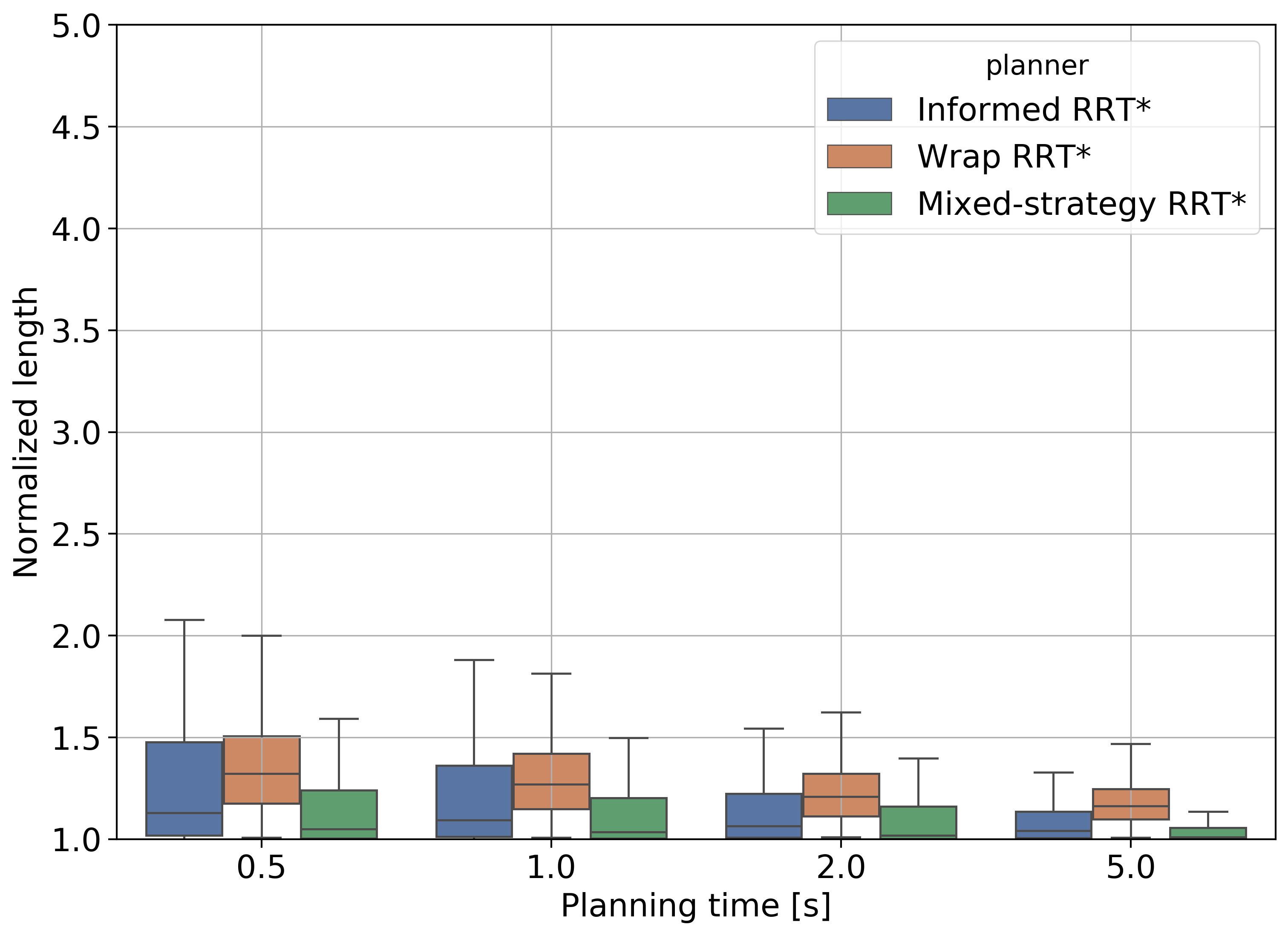} \label{fig: box_dof6}}\,%
	\subfloat[][$n=12$]
	{\includegraphics[height=4cm,width=0.32\textwidth]{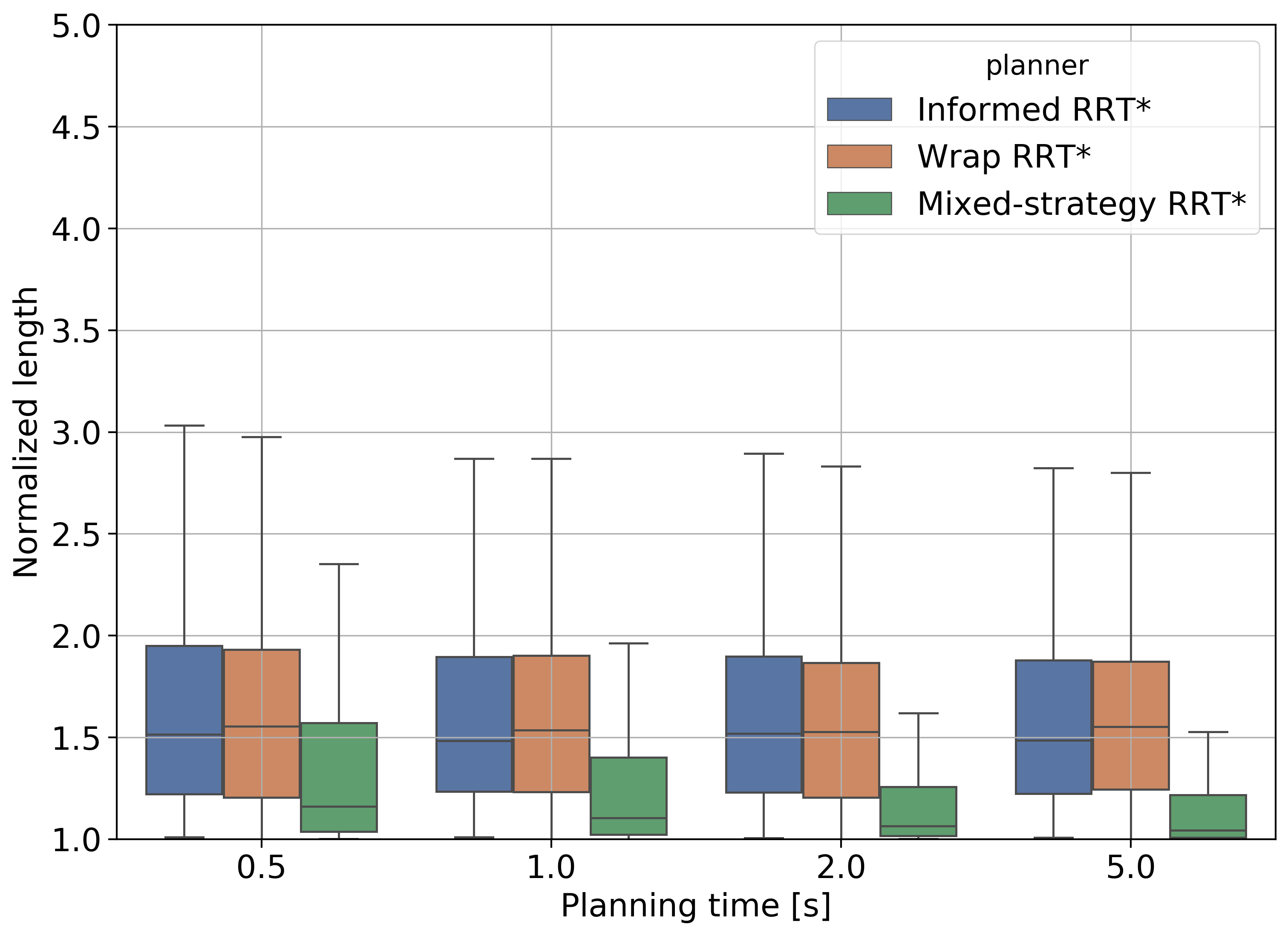} \label{fig: box_dof12}}\,
	\subfloat[][$n=18$]
	{\includegraphics[height=4cm,width=0.32\textwidth]{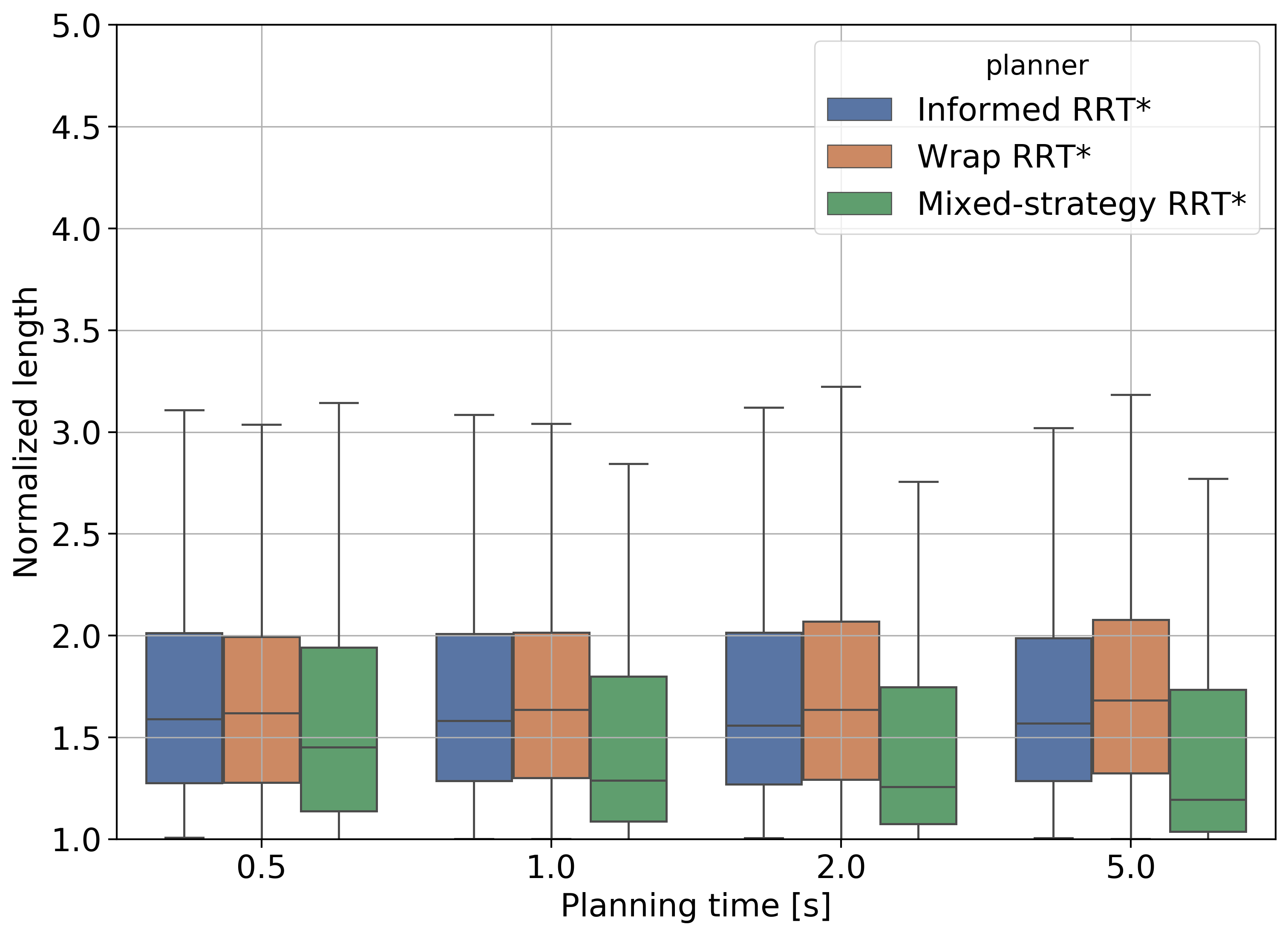} \label{fig: box_dof18}}%
	\caption{Robot manipulators (6-, 12-, 18-dof): Cost trends over planning time for randomized queries. }
	\label{fig: box-simulation}
\end{figure*}

\subsection{Mobile manipulators}
\label{sec: 2d-navigation}

We consider navigation scenarios with mobile manipulators. 
Each robot consists of a 6-degree-of-freedom manipulator mounted on an omnidirectional mobile platform (two linear and one rotational degree of freedom).
In the first case, the robot has to move from one side to the other of a wall with a narrow opening (see Figure \ref{fig: imm-scenario}). 
The problem has at least three homotopy classes: two circumnavigate the wall, and one passes through the narrow passage, requiring the re-configuration of the robot to fit the passage.
We compare our MI-RRT$^*$ with Informed-RRT$^*$ \cite{Gammel-InformedRRT} and Wrap-RRT$^*$ \cite{Kim-Song-wrapping-informed-rrt-2018} over 30 repetitions.
Results are in Figure \ref{fig: imm1}: MI-RRT$^*$ has the best convergence rate, followed by Wrap-RRT$^*$, and Informed-RRT$^*$.

We also consider a second scenario with two mobile manipulators (for a total of 18 degrees of freedom) required to move from one side to the other of a wall with two openings. 
Results are in Figure \ref{fig: imm2}: similarly to the single-robot case, MI-RRT$^*$ has the best convergence rate, followed by Wrap-RRT$^*$, and Informed-RRT$^*$, despite the greater number of iterations required by all methods to solve the problem.

\begin{figure*}[tbp]
	\centering
    \subfloat[][Mobile manipulator setup (single-robot)]
	{\label{fig: imm-scenario}
    \includegraphics[trim = 0cm 0cm 0cm 0cm, height=4cm,width=0.32\textwidth]{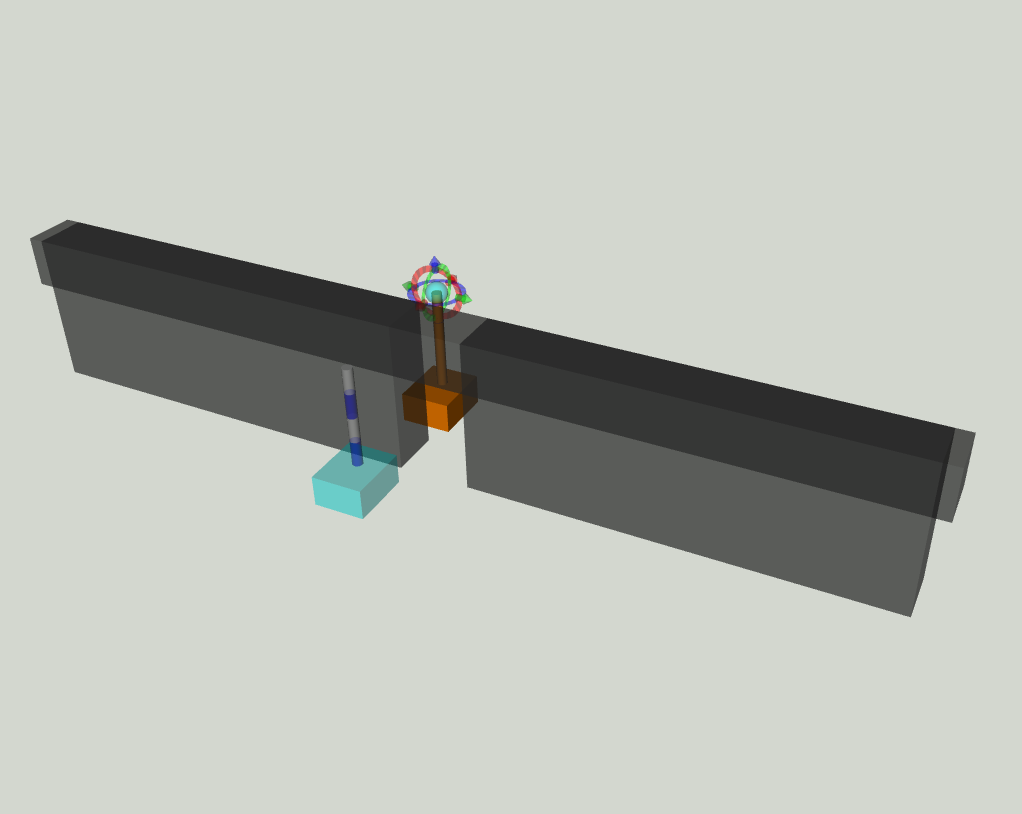} }
    \,%
	\subfloat[][Cost trend (single-robot)]
	{\label{fig: imm1}
    \includegraphics[height=4cm,width=0.32\textwidth]{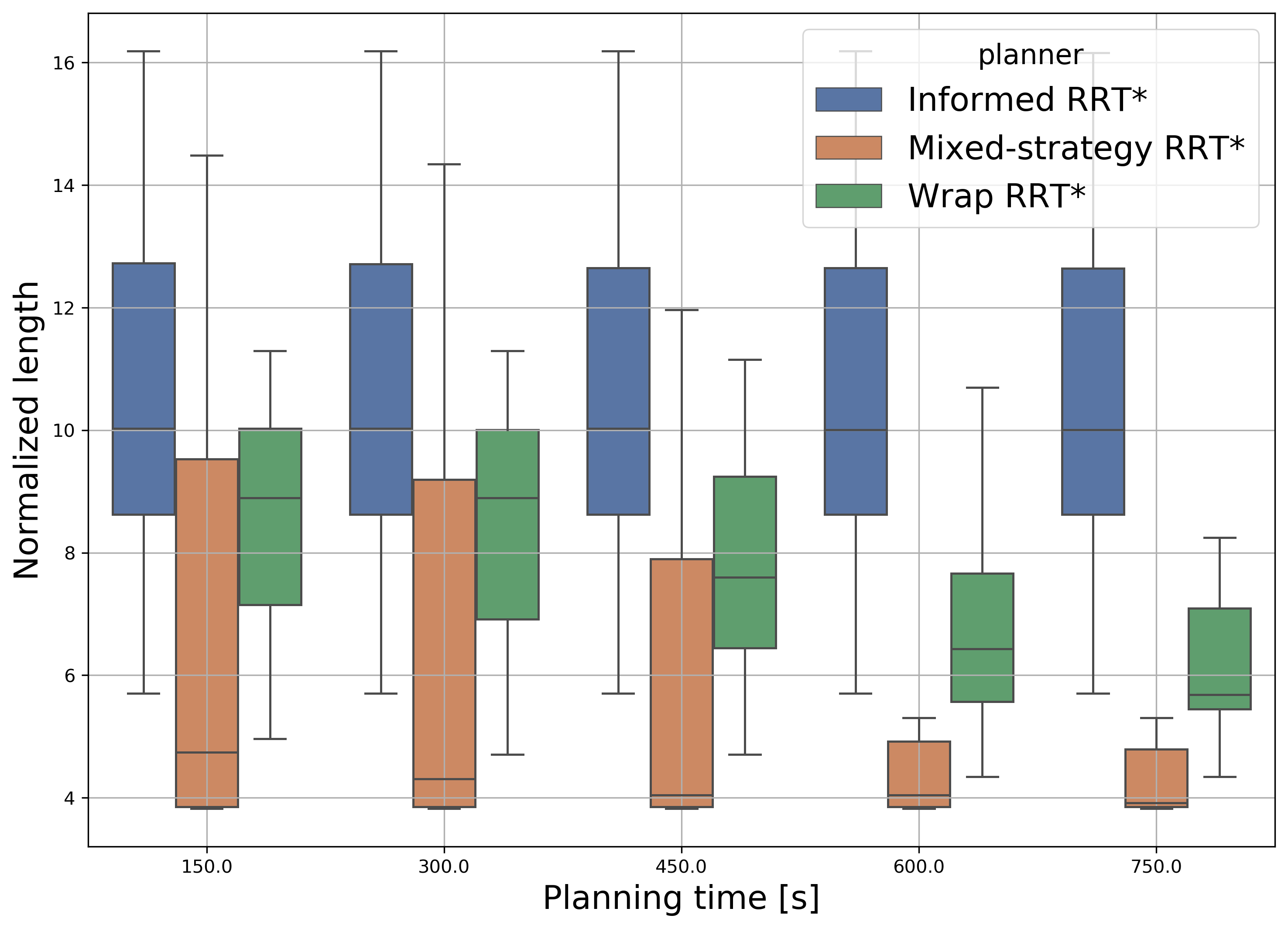} }
    \,%
	\subfloat[][Cost trend (two-robot)]
	{\label{fig: imm2}%
    \includegraphics[height=4cm,width=0.32\textwidth]{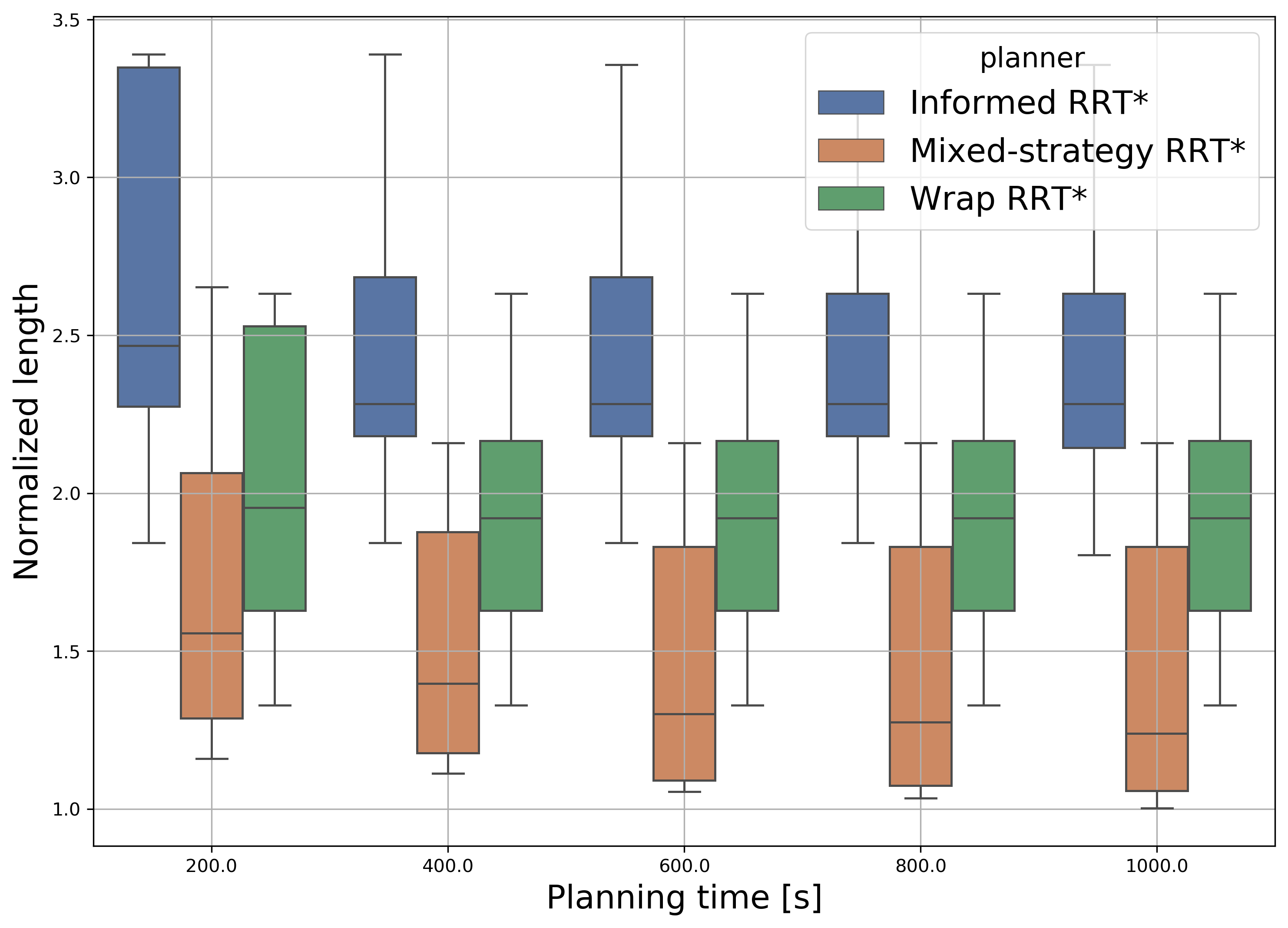}}
	\caption{Mobile manipulators: Cost trends over planning time with one and two robots.}
	\label{fig: imm}
\end{figure*}

\subsection{Real-world case study}
\label{sec: experiments}

\begin{figure*}[tpb!]
  \subfloat[]{\label{fig: cella-sharework.real}\includegraphics[trim = 25cm 0cm 30cm 0cm, clip, angle=0, height=5.5cm,width=0.4\textwidth]{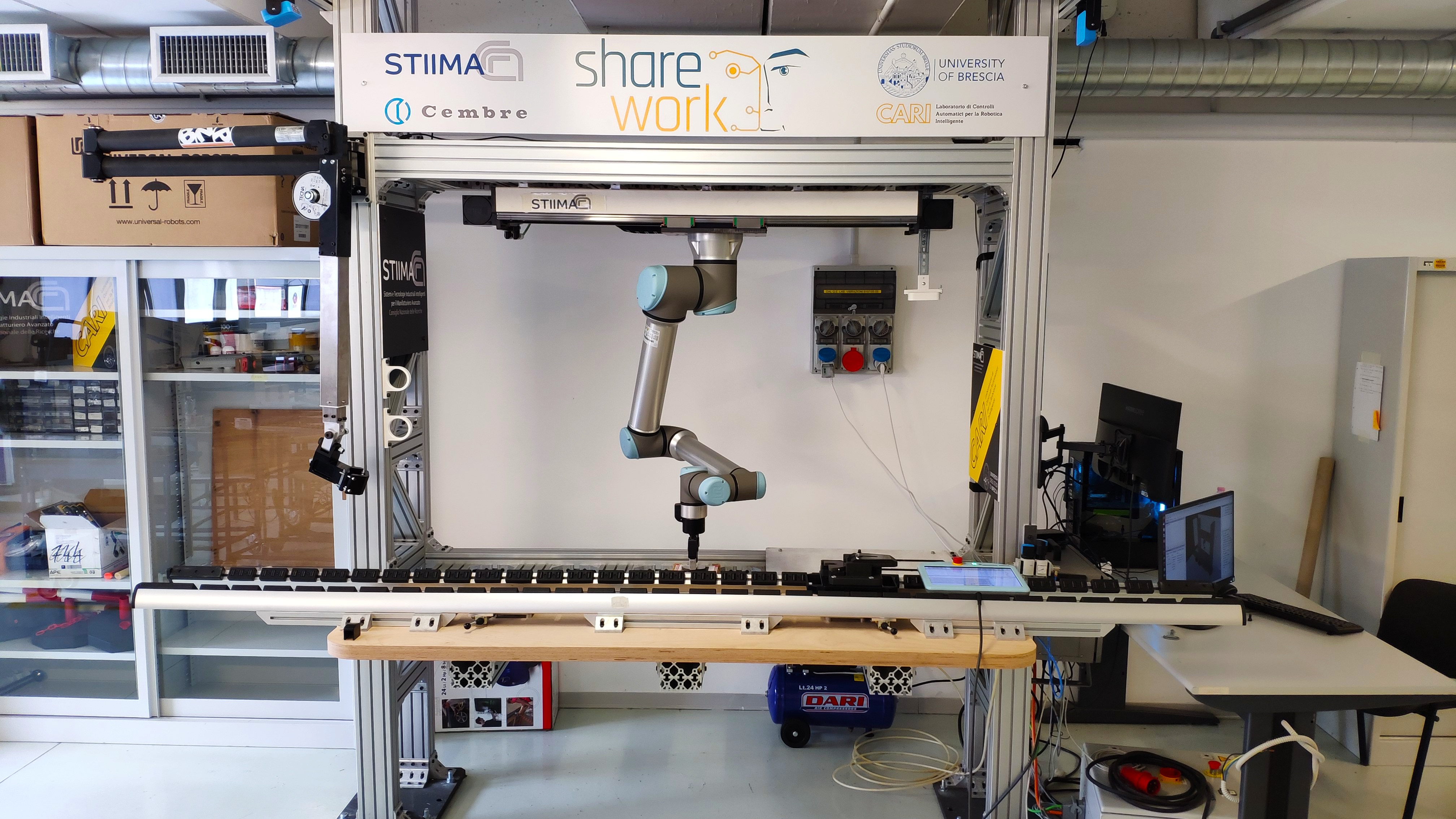}}\hfill
  \subfloat[]{\label{fig: cella-sharework-table}\includegraphics[trim = 12cm 0cm 12cm 0cm, clip, angle=0,height=5.5cm,width=0.2\textwidth]{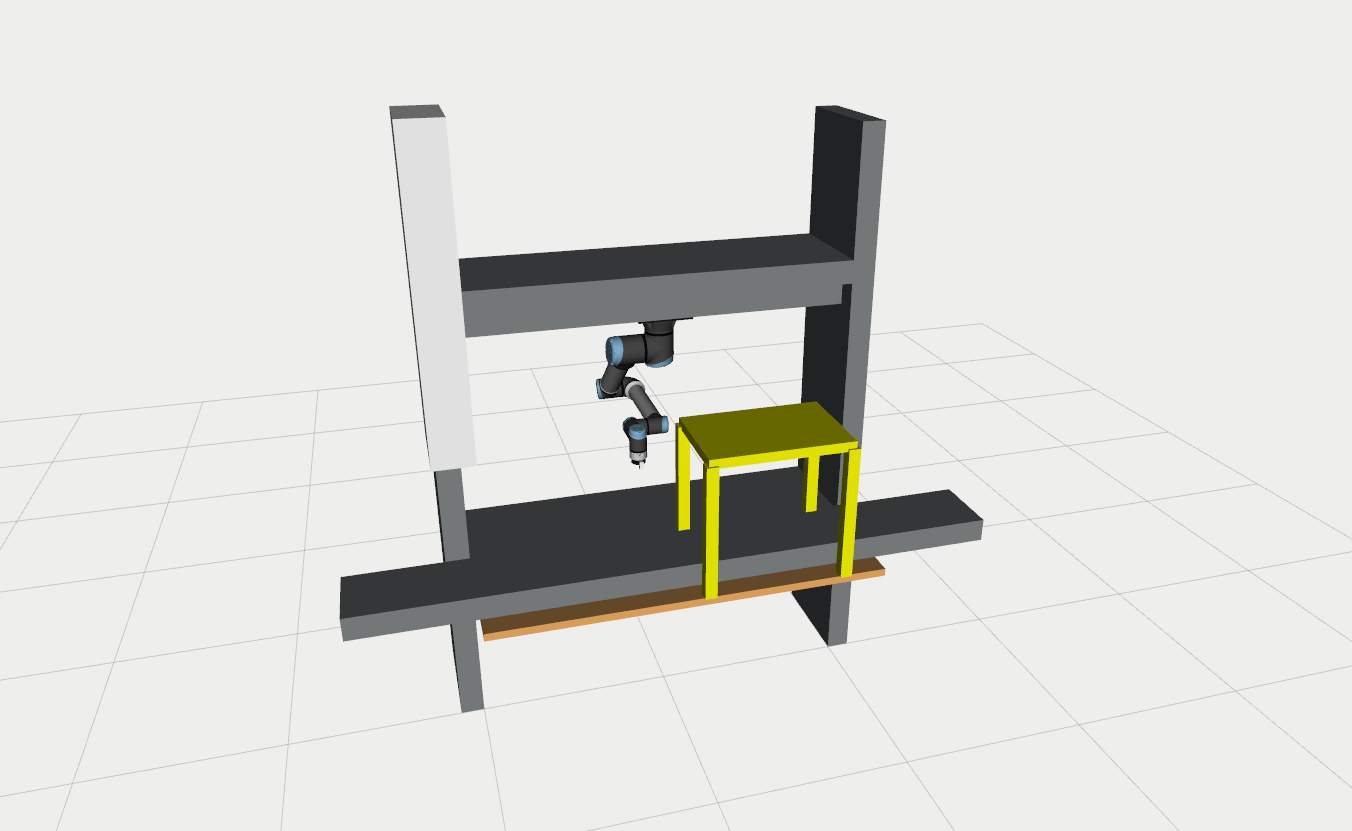}}\hfill
  \subfloat[]{\label{fig: cella-sharework-barrier}\includegraphics[trim = 12cm 0cm 12cm 0cm, clip, angle=0,height=5.5cm,width=0.2\textwidth]{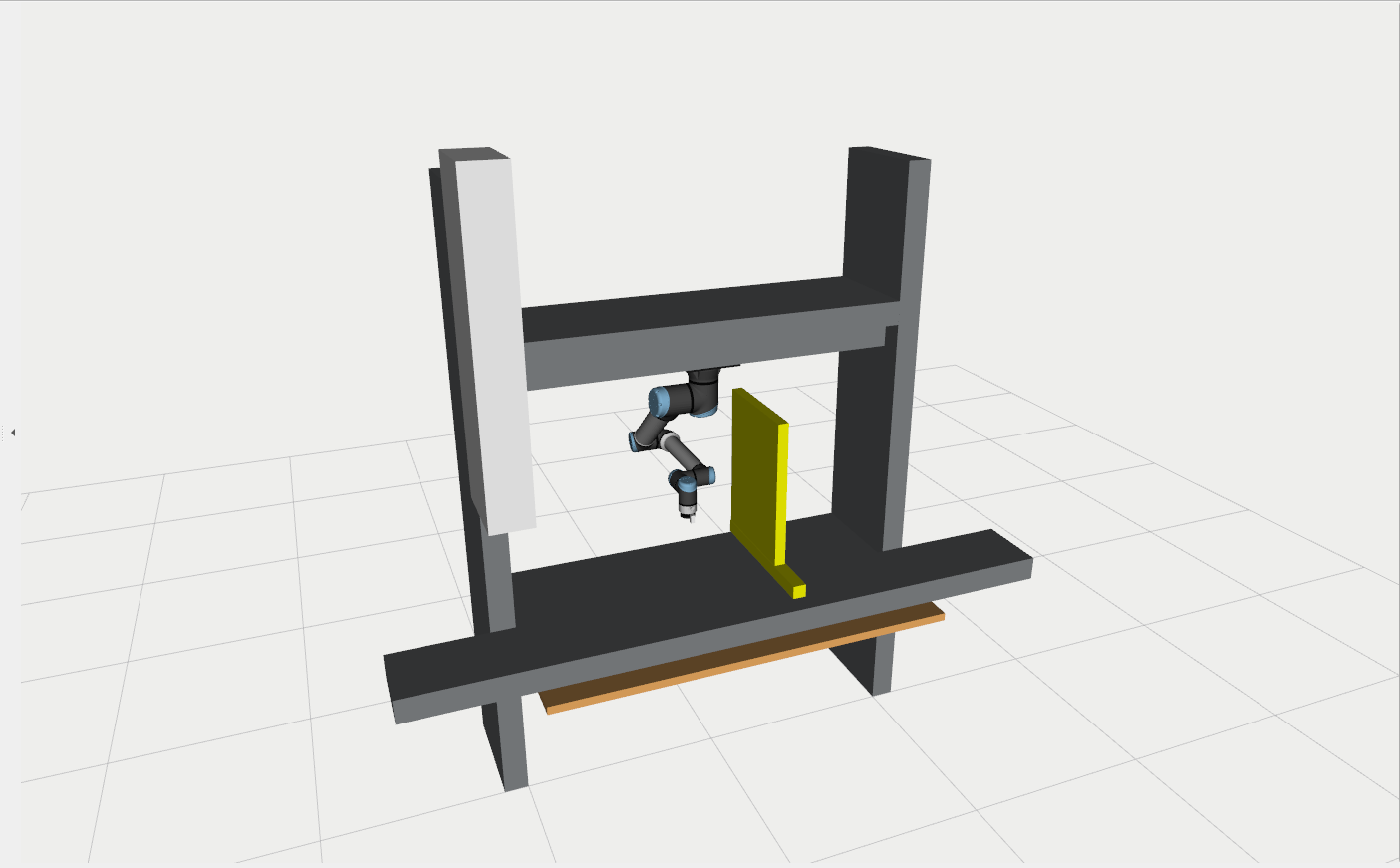}}\\
  \caption{Experimental setup. A Universal Robots UR10e mounted upside down works on the panel in front of it. (a) Actual setup; (b) first experiment: table-shaped obstacle upon placing position; (c) second experiment: barrier with a narrow passage between picking and placing positions.}
    \label{fig: cella-sharework}
\end{figure*}

We validated our algorithm in a manufacturing mock-up cell designed within the EU-funded project \emph{Sharework}. 
The cell consists of a $6$-degree-of-freedom collaborative robot, Universal Robots UR10e, mounted upside down and working on a work table in front of it (Figure \ref{fig: cella-sharework}). 
The proposed motion planner is implemented in C++ within ROS/MoveIt! \cite{Moveit}.
An open-source version of the code is available at \cite{git_mirrt}. 
ROS/MoveIt! runs on an external computer from which it sends the planned trajectory to the robot controller.

The robot is tasked with a sequence of fifty pick-and-place operations.
We consider two experiments. In the first one, a table-shaped obstacle is placed upon the placing goal (Figure \ref{fig: cella-sharework-table}). In the second one, a barrier separates the picking and placing goals (Figure \ref{fig: cella-sharework-barrier}). 
These scenarios simulate realistic machine-tending operations, in which the robot needs to access a confined space. 
From a planning perspective, they introduce narrow passages, complicating the planning problem. 
For example, in the barrier experiment, the shortest path passes through the narrow space below the barrier, close to the table surface.

Figure \ref{fig: box-experiment} compares the performance of MI-RRT$^*$, Informed-RRT$^*$, and Wrap-RRT$^*$ with different planning times, for the table and the barrier experiments. 
Similar to Sections \ref{sec: manipulators} and \ref{sec: 2d-navigation}, MI-RRT$^*$ has a faster convergence speed in both experiments. 
In the table experiment, MI-RRT$^*$ reduces the planning time up to $-34\%$ and $-13\%$ compared to Informed-RRT$^*$ and Wrap-RRT$^*$.
In the barrier experiment, MI-RRT$^*$ reduces the planning time up to $-37\%$ and $-18\%$ compared to Informed-RRT$^*$ and Wrap-RRT$^*$.
Note that, contrary to the simulations, Wrap-RRT$^*$ showed a significant improvement compared to Informed-RRT$^*$.
This suggests that the advantages of Wrap-RRT$^*$ are problem-dependent.

Overall, MI-RRT$^*$ finds better solutions with the same maximum planning time.
As a matter of example, Figure \ref{fig: box-experiment-long-time} shows the continuous trend of the normalized costs for the barrier experiment. 
The key result is that MI-RRT$^*$ approaches the best cost faster in the initial phase.
For example, after 10 seconds, MI-RRT$^*$ reaches $1.4 c^*$, while Informed-RRT$^*$ and Wrap-RRT$^*$ reach $1.6 c^*$ and $1.9 c^*$, respectively; after 60 seconds, MI-RRT$^*$ reaches $1.2 c^*$, while Informed-RRT$^*$ and Wrap-RRT$^*$ reach $1.35 c^*$ and $1.55 c^*$, respectively.

\begin{figure*}[tbp]
	\centering
	\subfloat[][table-shaped obstacle]
	{\includegraphics[height=4cm,width=0.3\textwidth]{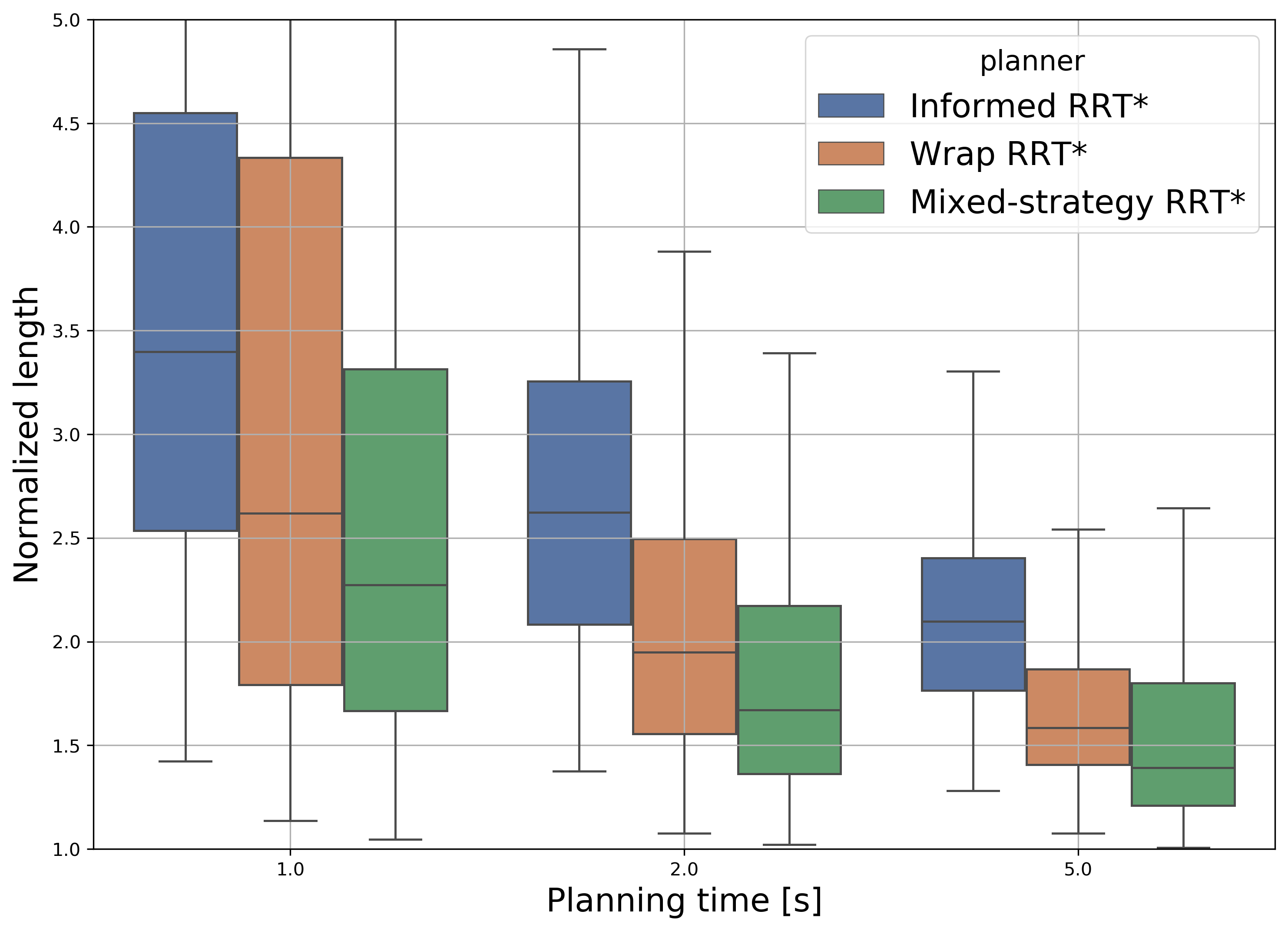} \label{fig: box_table}}\quad%
	\subfloat[][barrier obstacle]
	{\includegraphics[height=4cm,width=0.3\textwidth]{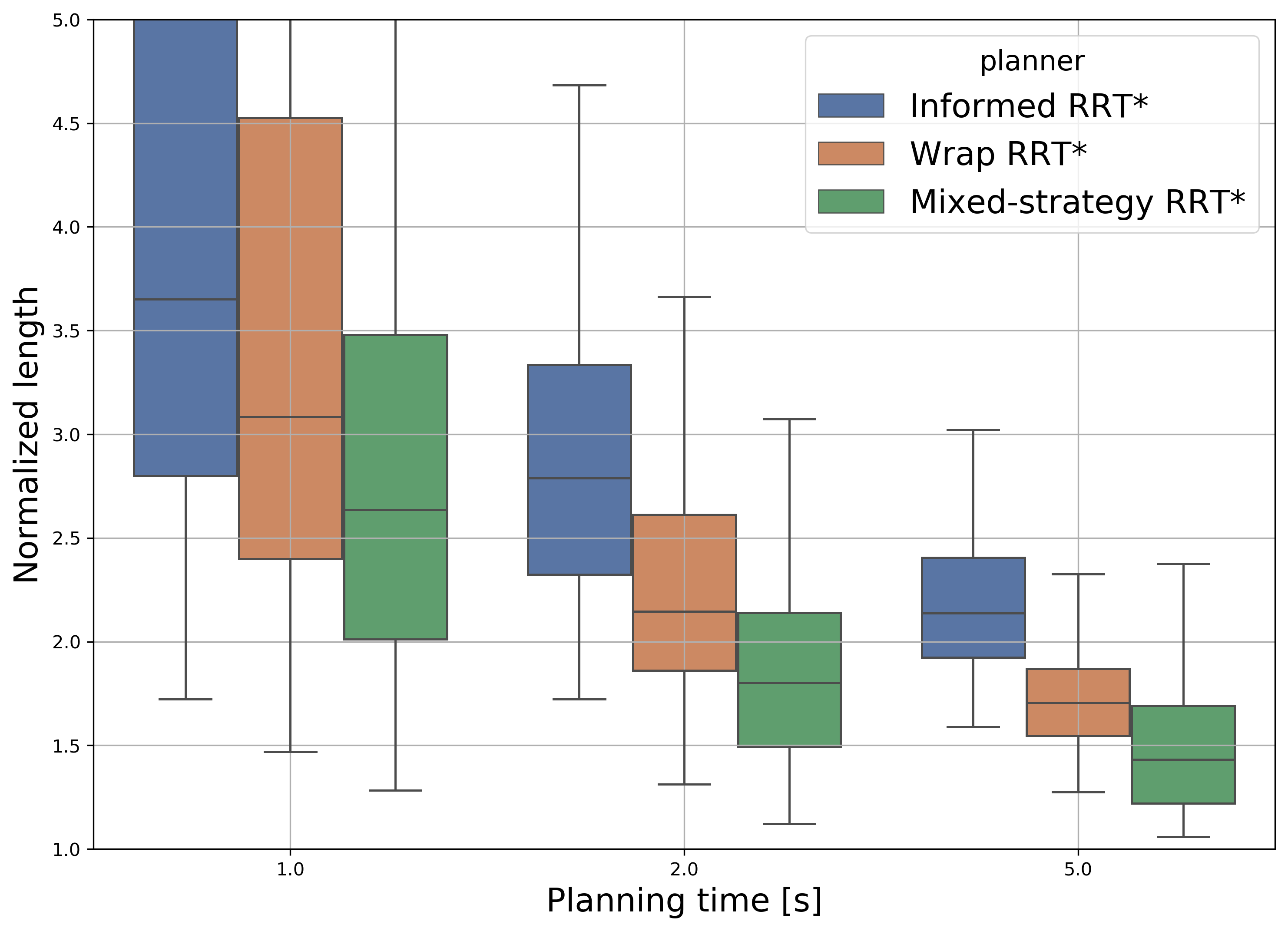} \label{fig: box_barrier}}\quad%
    \subfloat[][barrier obstacle -- long planning time]
	{\includegraphics[height=4cm,width=0.3\textwidth]{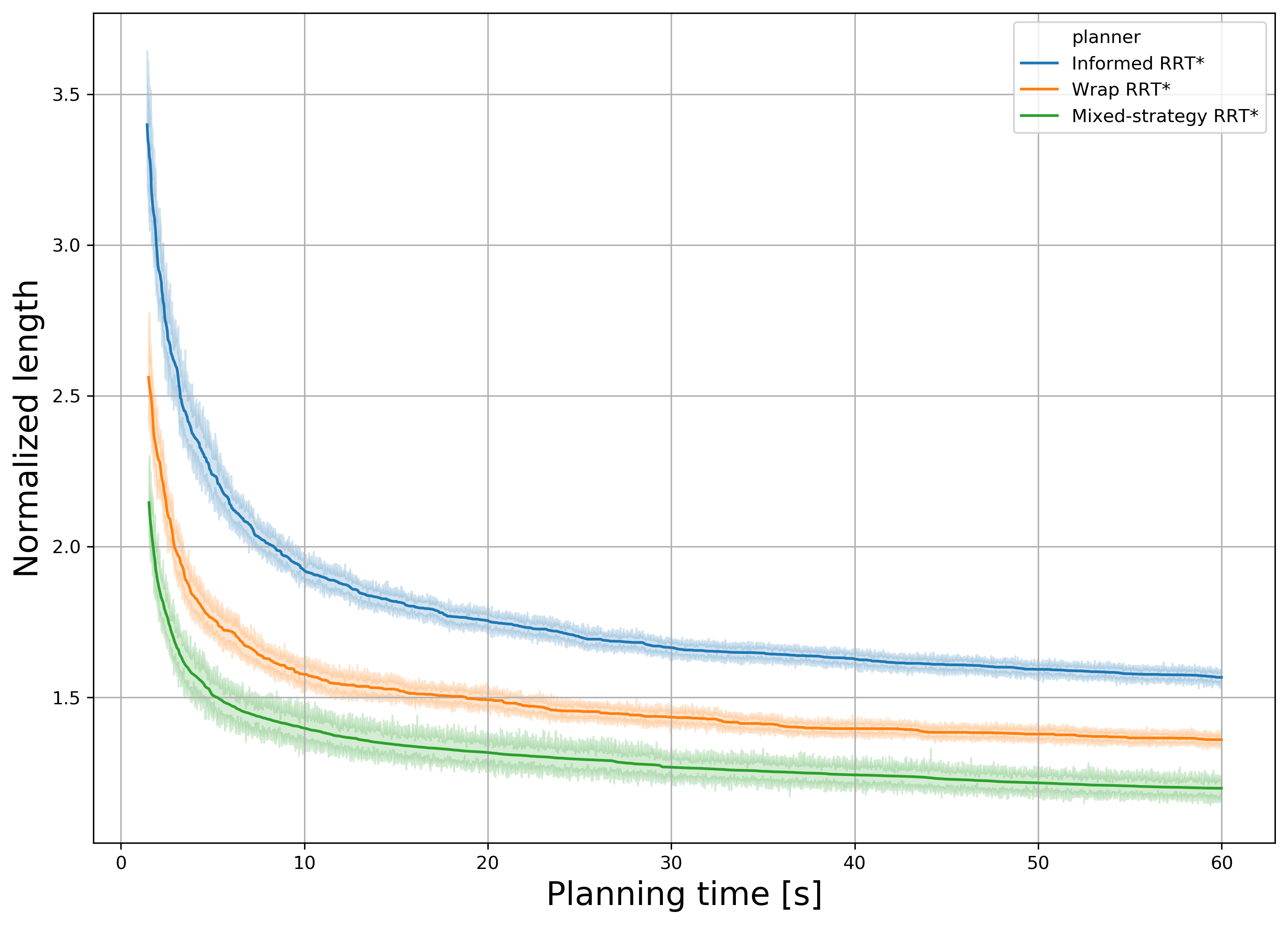} \label{fig: box-experiment-long-time}}\qquad%
	\caption{Experimental results: cost trends over planning time. 
    }
	\label{fig: box-experiment}
\end{figure*}

\section{Conclusions}
\label{sec: conclusions}

Comparisons with state-of-the-art methods highlight the effectiveness of the proposed method in improving the convergence speed, especially in high-dimensional problems. 
The method is implemented in a manufacturing-oriented case study, where the robot is tasked with a sequence of pick-and-place operations. 
Results show that the proposed planner converges quicker to the optimal solution, allowing for shorter planning latencies in online applications.

An open-source implementation of the algorithm is available at \cite{git_mirrt}. 
The algorithm is implemented in C++ and is fully compatible with ROS/MoveIt! \cite{Moveit}.
Examples of usage and benchmarking are also available at \cite{high_dof_cell_github}.

%\begin{appendix}
\begin{appendices}
	\section{Effect of the tuning parameters} \label{sec: Appendix}

We analyze the effect of the parameters $R_0$ and $\nu$ used in Algorithm \ref{alg: algorithm}.
To do so, we use an illustrative example consisting of a narrow-passage problem with one local minimum $c_\mathrm{local}$ and one global minimum $c_\mathrm{global}$. 
Different cardinalities of the configuration spaces are tested. 
We run  200 queries for each parameter set; each time, the algorithm runs for $10^6$ iterations with an early stop condition if the cost $c_k$ satisfies the condition $c_k<1.01 c_\mathrm{global}$. 
Although this analysis is limited to an illustrative example, the results can serve as tuning guidelines for parameters $R_0$ and $\nu$, as demonstrated in Section \ref{sec: results}.

\subsection{Narrow-passage example}
\label{sec: narrow-passage}

We consider the configuration space
\begin{equation}
    X=\left\{ x\in X \subseteq \mathbb R^n \;\big\vert\;  -5 \leq x_i \leq 5
    \right\}
\end{equation}
and an hollow hyper-spherindrical obstacle:
\begin{equation}\footnotesize
    X_{\mathrm{obs}}=\left\{ x\in X \;\big\vert\;  |x_1| \leq \frac{l_{c}}{2}, \;
     r_{c1}^2 \leq  \sum_{i=2}^{n} x_i^2  \leq r_{c2}^2
    \right\}
\end{equation}
where $l_c=1$ is length of the hyper-spherinder, $r_{c2}=1$ is the external radius, and $r_{c1}$ is the cavity radius.
The cavity radius is such that the ratio between the volume of the cavity and that of the external cylinder is equal to 0.5 for all values of $n$.
The starting and goal points are set equal to 
\[\small
x_{\mathrm{start}}=\left[-0.6l_c,a,0,...,0\right],\, x_{\mathrm{goal}}=\left[0.6l_c,a,0,...,0\right]
\]
with
\[
a=\frac{r_{c2}+3r_{c1}}{4}
\]
The problem has a local and a global minimum:
\[
\begin{aligned}
&c_\mathrm{local}=l_c+2\sqrt{(0.1l_c)^2+(r_{c1}-a)^2 }\\
&c_\mathrm{global}=l_c+2\sqrt{(0.1l_c)^2+(r_{c2}-a)^2 }.   
\end{aligned}
\]
An example of the planning problem for $n=2$ is in Figure \ref{fig: sferindro}.
\begin{figure}
    \centering
    \includegraphics[height=2.6cm,width=0.6\columnwidth]{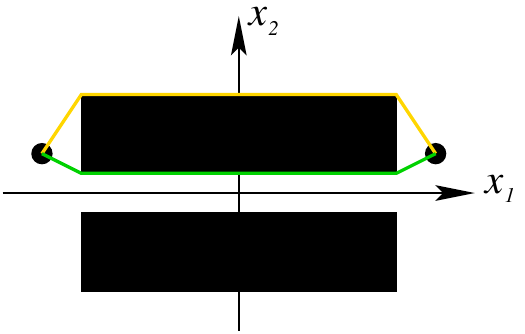}
    \caption{Narrow-passage example of Section \ref{sec: narrow-passage} for $n=2$. 
    The planning problem has one global optimum (green line) and one local optimum (yellow line). For the sake of readability, axis scales are not equal.}
    \label{fig: sferindro}
\end{figure}

\subsection{Effect of $R_0$}
\label{sec: narrow_pass_radius}

$R_0$ should be adequately small compared to the current cost. 
We run tests for $R_0 \in [10^{-3},\;10^{-1}]$ and $\nu = 0.999$.
For each test, we count the number of iterations needed to reach $c_k \leq 1.01 c_\mathrm{global}$. 
The 90\%-percentile, computed over 200 tests, is used as the performance index.
Figure \ref{fig: radius_narrow_pass_results} shows the performance obtained for different values of  $R_0=\frac{R}{c_k-u}$ and $n$. 
Values around $0.02$ provide the best results, while the local optimization is less effective with higher values. Smaller values of $R_0$ provide minimal improvements to the cost function.

\subsection{Effect of forgetting factor $\nu$}
\label{sec: narrow_pass_forgetting_factor}

The forgetting factor allows smoothing the switching between the two sampling strategies by averaging out the cost changes over multiple iterations. 
Figure \ref{fig: forgetting_factor_narrow_pass_results} shows the relation between the forgetting factor $\nu$ and the number of iterations required to reach $c_k=1.01 c_\mathrm{global}$ (90\%-percentile), the tube radius $R_0$ was set equal to $0.02$ according to Section \ref{sec: narrow_pass_radius}. 
If $\nu>0.999$, results do not vary significantly; however, selecting values of $\nu$ too close to 1 could lead the solver to get stuck in local minima for many iterations. $\nu=0.999$ is a reasonable value for most cases in our experience. 

\begin{figure*}[tbp]
	\centering
	\subfloat[a][Effect of radius $R_0$]
	{\includegraphics[height=0.2\textheight]{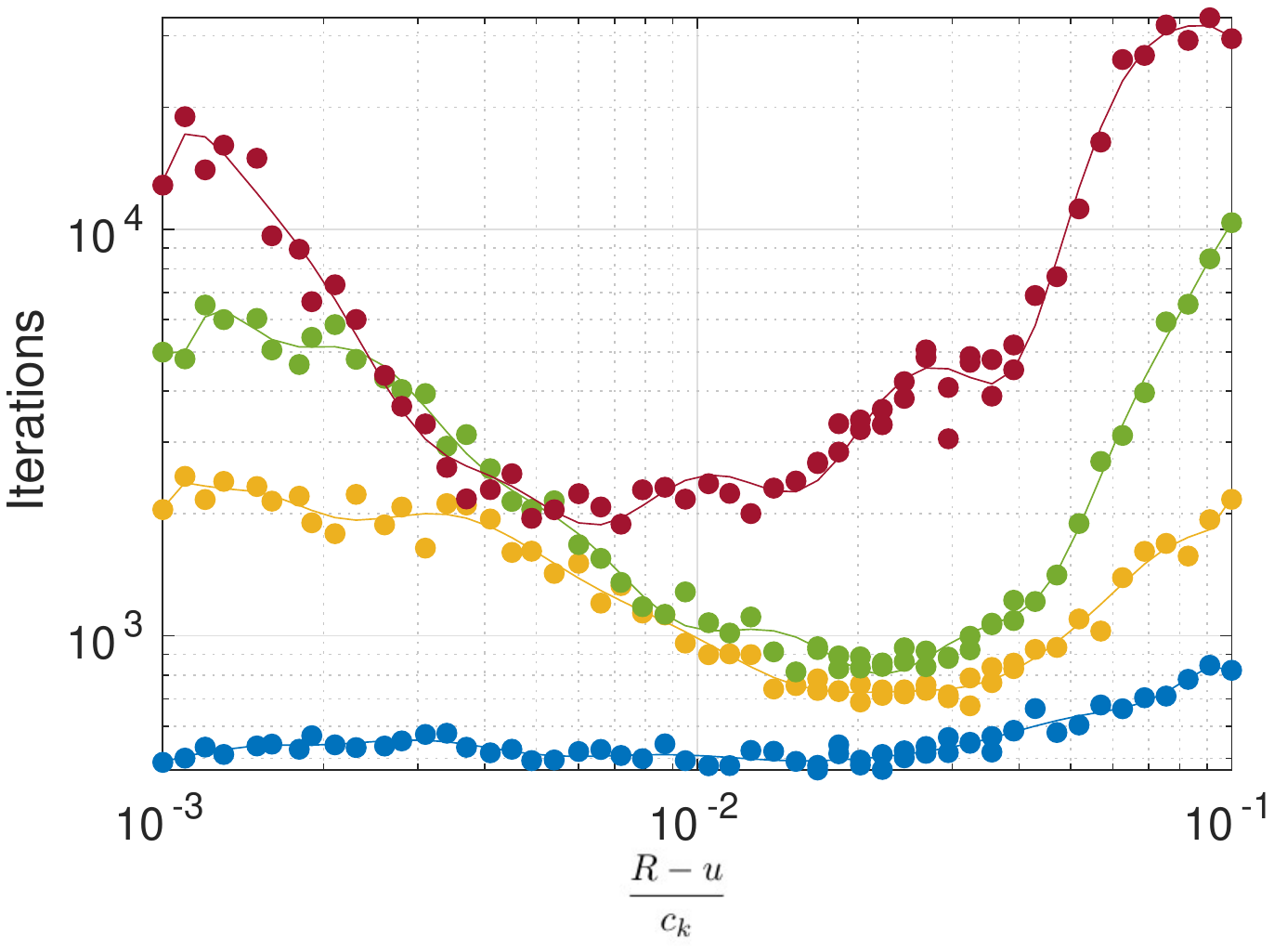} \label{fig: radius_narrow_pass_results}}\quad%
	\subfloat[b][Effect of forgetting factor $\nu$]
	{\includegraphics[height=0.2\textheight]{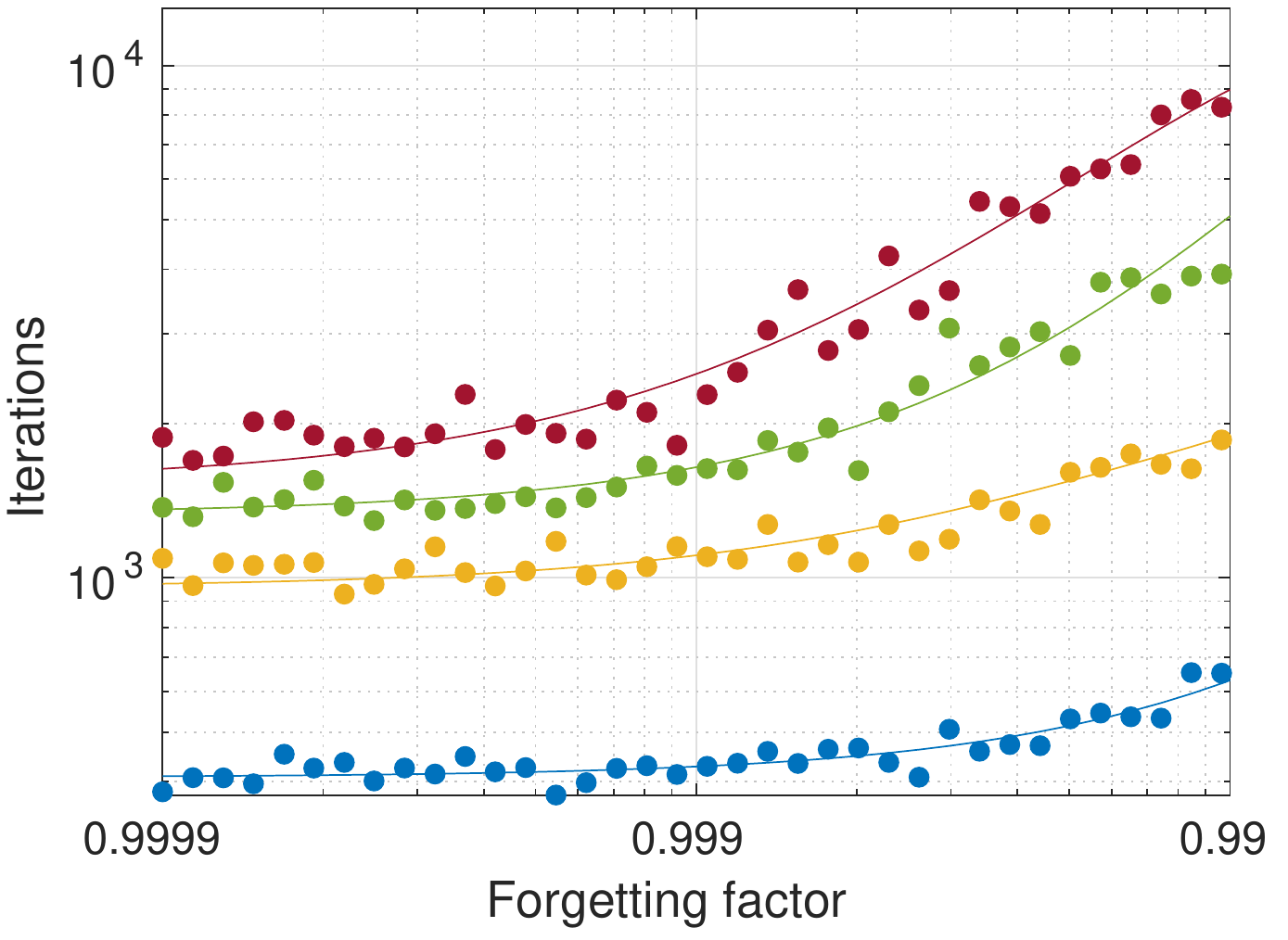}\label{fig: forgetting_factor_narrow_pass_results}}\,\,%
	\caption{Effect of the parameters $R_0$ and $\nu$. Iterations needed to reach $c_k=1.01 c_\mathrm{global}$ (90\%-percentile from 200 tests) for different values of $R_0$ and $\nu$. Blue line: $n=2$; orange line: $n=3$; green line: $n=4$, red line: $n=7$.}
	\label{fig: role-of-parameters}
\end{figure*}

\end{appendices}

\ifnum\value{template}=1
    \backmatter
    
    \section*{Declarations}
    
    \bmhead{Ethical Statement}
    
    The authors declare that the following is fulfilled: 
    1) This material is the authors' original work, which has not been previously published elsewhere;
    2) The paper is not currently being considered for publication elsewhere;
    3) The paper reflects the authors' research and analysis truthfully and completely;
    4) The paper properly credits the meaningful contributions of co-authors and co-researchers;
    5) All authors have been personally and actively involved in substantial work leading to the paper, and will take public responsibility for its content.
    
    \bmhead{Authors' contribution}
    M.F. and M.B. devised the methodology and wrote the main manuscript text. All authors contributed to implementing the software and conceiving the experiments. All authors read and reviewed the manuscript.
    
    \bmhead{Funding}
    This study was partially carried out within the MICS (Made in Italy – Circular and Sustainable) Extended Partnership and received funding from Next-Generation EU (Italian PNRR – M4 C2, Invest 1.3 – D.D. 1551.11-10-2022, PE00000004). CUP MICS D43C22003120001.
\fi

\bibliographystyle{IEEEtran}
\bibliography{reference_stiima}

\end{document}